\pdfoutput=1

\documentclass[11pt]{article}

\usepackage{ACL2023}

\usepackage{bbm}
\usepackage{todonotes}
\makeatletter
\newcommand*\iftodonotes{\if@todonotes@disabled\expandafter\@secondoftwo\else\expandafter\@firstoftwo\fi}  %
\makeatother

\definecolor{dandelion}{HTML}{FFD464}

\definecolor{bittersweet}{HTML}{C04F17}

\definecolor{mintgreen}{RGB}{152, 255, 152}

\usepackage{tipa}
\newcommand{\ethz}{\text{\normalfont \textipa{D}}}
\newcommand{\uchic}{\normalfont \text{\textipa{@}}}
\newcommand{\cambridge}{\normalfont \text{\textipa{N}}}
\newcommand{\copenhagen}{\normalfont \text{\textipa{R}}}

\usepackage{times}
\usepackage{latexsym}
\usepackage{siunitx}
\usepackage{xspace}
\usepackage{enumitem}
\usepackage[multiple]{footmisc}
\usepackage{xcolor}

\usepackage[T1]{fontenc}

\usepackage[utf8]{inputenc}

\usepackage{microtype}

\usepackage{inconsolata}

\usepackage[noend]{algpseudocode}
\usepackage{algorithm}
\usepackage{amsfonts}
\usepackage{amsmath}
\usepackage{amssymb}
\usepackage{bm}
\usepackage{cleveref}
\usepackage{xspace}
\usepackage{amsthm}
\usepackage{mathtools}
\usepackage{relsize}
\usepackage{caption,subcaption,booktabs}
\usepackage{import}
\usepackage[htt]{hyphenat}
\usepackage{booktabs}
\usepackage{graphicx}
\usepackage{tikz}
\usepackage{tabularx}
\usetikzlibrary{positioning}
\usetikzlibrary{backgrounds}
\usetikzlibrary{calc}

\input{algos/algonames.tex}

\crefname{section}{\S}{\S\S}
\Crefname{section}{\S}{\S\S}
\crefname{table}{Tab.}{}
\crefname{figure}{Fig.}{}
\crefname{algorithm}{Algorithm}{}
\crefname{equation}{Eq.}{Eqs.}  %
\crefname{appendix}{App.}{}
\crefname{thm}{Theorem}{Theorems}
\crefname{prop}{Proposition}{Propositions}
\crefname{cor}{Corollary}{Corollaries}
\crefname{observation}{Observation}{Observations}
\crefname{assumption}{Assumption}{Assumptions}
\crefformat{section}{\S#2#1#3}

\theoremstyle{definition}

\theoremstyle{definition}

\newtheorem{cor}{Corollary}

\newcommand{\defn}[1]{{\textbf{#1}}}

\newcommand{\defequals}{\triangleq}
\newcommand{\mtrian}{\mathrel{\raisebox{-0.1ex}{%
\scalebox{0.8}[0.6]{$\vartriangle$}}}}
\newcommand{\defpropto}{\overset{\mtrian}{\propto}}

\newcommand{\ent}{\mathrm{H}}
\usetikzlibrary{shapes.misc}

\def\calX{{\mathcal{X}}}
\def\calY{{\mathcal{Y}}}

\definecolor{MyTawny}{HTML}{d55e00} %
\definecolor{MyGreen}{HTML}{029e73}
\definecolor{MyBlue}{HTML}{0173b2}
\definecolor{MyOrange}{HTML}{de8f05}
\definecolor{MyBronze}{HTML}{ca9161}
\definecolor{MySilver}{HTML}{949494}
\definecolor{MyRed}{HTML}{b40426}
\definecolor{MyInsignificantBlue}{HTML}{3b4cc0}

\newcommand{\querytext}[1]{{\color{MyGreen} \textit{#1}}}
\newcommand{\contexttext}[1]{{\color{MyOrange} \textit{#1}}}
\newcommand{\entitytext}[1]{{\color{MyBlue}{\textit{#1}}}}
\newcommand{\entityexample}[1]{{\color{MyBlue}{\textit{\textbf{#1}}}}}
\newcommand{\realtext}[1]{{\color{MyBronze} \textit{#1}}}
\newcommand{\faketext}[1]{{\color{MySilver} \textit{#1}}}
\newcommand{\othertext}[1]{{\color{MyGreen} \textit{#1}}}
\newcommand{\opentext}[1]{{\color{MyGreen} \textit{#1}}}
\newcommand{\closedtext}[1]{{\color{MyTawny} \textit{#1}}}

\newcommand{\pscore}{\emph{persuasion score}\xspace}

\newcommand{\stickyscore}{\emph{susceptibility score}\xspace}

\newcommand{\context}{{\color{MyOrange}{c}}}

\newcommand{\rvContext}{{\color{MyOrange}{C}}}
\newcommand{\rvAnswer}{{\color{MyTawny}{A}}}
\newcommand{\rvEntity}{{\color{MyBlue}{E}}}

\newcommand{\alphabet}{\Sigma}
\newcommand{\entity}{{\color{MyBlue}{e}}}
\newcommand{\entities}{{\color{MyBlue}{\mathcal{E}}}}
\newcommand{\lm}{p_{\textsc{m}}}
\newcommand{\corpus}{\mathcal{D}}
\newcommand{\answer}{{\color{MyTawny}a}}
\newcommand{\query}{{\color{MyGreen}{q}}}
\newcommand{\queries}{{\color{MyGreen}{Q}}}
\newcommand{\stickyscoresymb}{\chi}
\newcommand{\pscoresymb}{\psi}
\newcommand{\confusionscoresymb}{\kappa}

\newcommand{\hpmi}{\mathrm{HPMI}}

\newcommand{\entropy}{\mathrm{H}}
\newcommand{\MI}{\mathrm{I}}

\DeclareMathOperator*{\expect}{\mathbb{E}}

\definecolor{red_fig}{HTML}{D95847}
\definecolor{blue_fig}{HTML}{5D7CE6}

\newcommand{\numRelations}{122\xspace}

\newtheorem{prop}{Proposition}
\crefname{prop}{Proposition}{}

\title{Context versus Prior Knowledge in Language Models}

\author{
Kevin Du$^{\ethz}$~\;~
Vésteinn Snæbjarnarson$^{\copenhagen}$~\;~ 
Niklas Stoehr$^{\ethz}$~\;~
\\
\textbf{Jennifer C. White}$^{\cambridge}$~\;~
\textbf{Aaron Schein}$^{\uchic}$~\;~
\textbf{Ryan Cotterell}$^{\ethz}$
\\
$^{\ethz}$ETH Z{\"u}rich \quad $^{\copenhagen}$University of Copenhagen \quad \\
$^{\cambridge}$University of Cambridge \quad $^{\uchic}$The University of Chicago
\\
\href{mailto:kevidu@ethz.ch}{\texttt{kevin.du@inf.ethz.ch}}~\;~ 
\href{mailto:vesteinn.snaebjarnarson@gmail.com}{\texttt{vesn@di.ku.dk }}~\;~ 
\href{mailto:niklas.stoehr@inf.ethz.ch}{\texttt{niklas.stoehr@inf.ethz.ch}}\\
\href{mailto:jw2088@cam.ac.uk}{\texttt{jw2088@cam.ac.uk}}~\;~ 
\href{mailto:schein@uchicago.edu}{\texttt{schein@uchicago.edu}}~\;~ 
\href{mailto:ryan.cotterell@inf.ethz.ch}{\texttt{ryan.cotterell@inf.ethz.ch}}
}

\begin{document}
\maketitle
\begin{abstract}
To answer a question, language models often need to integrate prior knowledge learned during pretraining and new information presented in context.
We hypothesize that models perform this integration in a predictable way across different questions and contexts: models will rely more on prior knowledge for questions about entities (e.g., persons, places, etc.) that they are more familiar with due to higher exposure in the training corpus, and be more easily persuaded by some contexts than others.
To formalize this problem, we propose two mutual information-based metrics to measure a model's dependency on a context and on its prior about an entity: first, the \pscore of a given context represents how much a model depends on the context in its decision, and second, the \stickyscore of a given entity represents how much the model can be swayed away from its original answer distribution about an entity.
We empirically test our metrics for their validity and reliability.
Finally, we explore and find a relationship between the scores and the model's expected familiarity with an entity, and provide two use cases to illustrate their benefits.

\vspace{0.5em}
\hspace{.5em}\includegraphics[width=1.25em,height=1.15em]{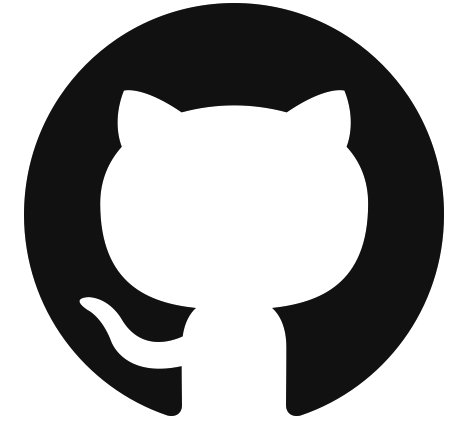}\hspace{.75em}
\parbox{\dimexpr\linewidth-7\fboxsep-7\fboxrule}{\url{https://github.com/kdu4108/measureLM}}
\vspace{-.5em}
\end{abstract}

\section{Introduction}
\label{sec:introduction}
Language models have displayed remarkable abilities to answer factual queries about entities, suggesting that they encode knowledge about these entities learned during pretraining \citep{petroni_language_2019, brown_language_2020, roberts_how_2020, geva_transformer_2021}.
For prompts that extend a question with additional information or context, the model can draw on both its prior knowledge and the additional context to answer the query \citep{kwiatkowski_natural_2019, joshi_triviaqa_2017, berant_semantic_2013, kasai_realtime_2022}.
While previous research has investigated how often a model will rely on prior knowledge over conflicting contextual information in answering questions \citep{longpre_entity-based_2022}, we hypothesize that models will not behave identically for all contexts and entities.
For example, if a language model is prompted with \contexttext{Harry hugged Voldemort.} \querytext{How friendly are \entityexample{Harry Potter} and \entityexample{Lord Voldemort}?}, we might expect the prior knowledge learned from training data describing the rivalry between these two characters to significantly influence the model's answer.
However, if the model lacks a strong prior on, say, \entityexample{Susie} and \entityexample{Alia}, then we might expect its answer to be primarily context-driven when prompted with \contexttext{Susie hugged Alia.} \querytext{How friendly are \entityexample{Susie} and \entityexample{Alia}?}.

\begin{figure}[t]
    \centering    \includegraphics[width=\columnwidth]{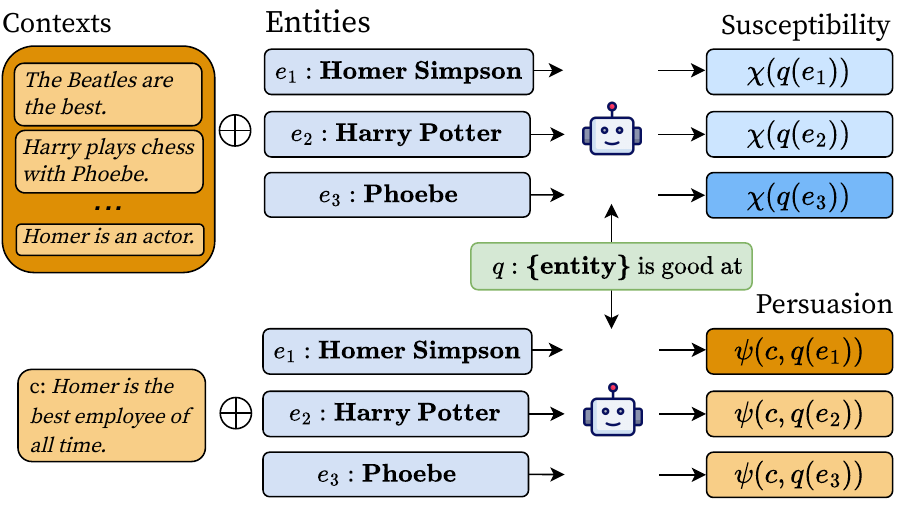}
    \vspace{-20pt}
    \caption{
    In answering a given query, a model may be more \emph{\entitytext{susceptible}} to context for some \entitytext{entities} than others, while some \contexttext{contexts} may be more \emph{\contexttext{persuasive}} than others (as indicated in this figure by color darkness in the rightmost column).
    We introduce mutual information-based metrics to evaluate how much impact the context has relative to the prior knowledge of a model.}
    \label{fig:model-valid}
\vspace{-15pt}
\end{figure}

We formalize this problem through the lens of evaluating the change in a model's answer distribution for different contexts and entities.
We present two mutual information-based metrics that allow us to explore differences in the effect of specific contexts on model behavior for different entities. 
The \defn{persuasion score} of a given context measures how much a model's answer distribution is affected by the context when prompted with a particular query about a given entity.
The \defn{susceptibility score} of a given entity measures how much the model's answer distribution can be swayed for a particular query about that entity, marginalized over all contexts.
Given their basis in mutual information, our metrics are natural operationalizations of persuasion and susceptibility. 
Furthermore, we offer empirical evidence of the validity and reliability of these measures by comparing them against similar metrics and showing their robustness to paraphrases and different samples.

To study how language models behave for different contexts and entities, we create a synthetic dataset of queries covering \numRelations topic areas extracted from the YAGO knowledge graph \citep{yago2007}, entities extracted from YAGO and generated with GPT-4 \cite{openai_gpt4_2023}, and contexts constructed with different qualities, e.g., relevancy and assertiveness. 
We apply our new metrics to Pythia models ranging from 70m to 12b parameters \citep{biderman2023pythia} to find evidence that relevant contexts are consistently more persuasive than irrelevant ones, and assertive contexts are more persuasive than less assertive ones for yes--no questions.
In a deep dive into one model, we find evidence that entities with high expected familiarity, as measured by both training data frequency and entity degree statistics in the YAGO knowledge graph, have lower susceptibility scores.

We further conduct case studies to show how these metrics could be useful in applied settings.
In a study on friend--enemy stance measurement, we find evidence that enemy duos are less susceptible than friend duos.
Applying our metrics to gender bias, we find evidence for a difference in susceptibility between stereotypically masculine and feminine names for gender-biased contexts. 
Through this, we show how our proposed metrics can be used to better analyze the effects of context and prior knowledge, with the potential for application toward greater control over model behavior.

\section{Context and Prior Knowledge}
\label{sec:related-work}

Much prior work has noted that language models develop biases about entities during training and investigated the tension between this entity bias and additional information provided in the context.

\subsection{Entity Bias}
Loosely defined across studies as model bias from spurious correlations between entity mentions and a target characteristic in the training data, \defn{entity bias} has been mainly examined in relation extraction (RE), where a model extracts relationships between entities from text \citep{zelenko_kernel_2002, zhou_exploring_2005}.
Different studies on RE have attempted to either mitigate \citep{zhang_2018_tacred, zhang_graph_2018, peng_learning_2020, wang_should_2022} or leverage \citep{yamada_2020_luke, zhou_chen_2022} entity bias for improved performance on the task.
\citet{wang_should_2022} note that entities can carry both useful semantic information about their roles, e.g., whether the entity is a person or a place, and spurious information which can bias the model toward relations not mentioned in the target sentence, e.g., the model may link \entityexample{Switzerland} with the relation \querytext{countries\_of\_residence} at inference time for a sentence that does not mention this relation, due to their frequent co-occurrence in the training data.
Additional studies on machine reading comprehension (MRC) tasks, e.g., SQuAD \citep{rajpurkar_etal_2016_squad}, TriviaQA \citep{joshi_triviaqa_2017}, and NaturalQuestions \citep{kwiatkowski_natural_2019}, have found substituting different entity names can result in meaningful changes in model predictions and overall evaluation performance \citep{yan_robustness_2022}.\looseness=-1

\subsection{Context and Entity Bias}
The existence of entity bias naturally raises the question of how it interacts with context to shape a model's response.
Several papers \citep{longpre_entity-based_2022, chen-etal-2022-rich, xie_adaptive_2023} approach this by inducing and exploring \defn{knowledge conflicts}, where a context preceding a query proposes information that conflicts with a model's prior knowledge about that query.
They measure the model's reliance on pretrained entity bias and context by computing the \defn{memorization ratio}: the proportion of knowledge conflict examples for which the model maintains its answer from prior knowledge.
\citet{pezeshkpour2023measuring} measures a model's prior knowledge of a fact by comparing the entropy of a queried model's answer distribution before and after stating the fact in context.
Several studies further propose interventions to reduce entity bias and favor in-context information.
These approaches include prompting 
\citep{zhou_context-faithful_2023, onoe_can_2023}, modifying training data \citep{wang_causal_2023}, fine-tuning \citep{li_large_2022}, and neuron-level interventions \citep{yu_characterizing_2023} at inference time.\looseness=-1

However, the metrics used to quantify model reliance on context and entity bias in these papers---excepting \citet{pezeshkpour2023measuring}---are limited in several ways. 
First, most previous work does not develop \emph{entity-specific} or \emph{context-specific} metrics for the strength of the entity bias or context persuasiveness. 
Instead, their metrics produce only a single number to summarize the model's overall reliance on entity bias.
Second, their setups are limited to adversarial cases in which a context is chosen to go against the entity bias. 
Indeed, we may well wish to measure the interplay of context and entity bias in other cases, such as when context reinforces entity bias or when they do not clearly disagree.
Therefore, we seek to define a metric that measures how much the model depends on a given context or entity with rigorous, theoretically grounded interpretations.\looseness=-1

\section{Our Formalization} 
\label{sec:formalization}

We now formalize the problem setting in which we define metrics for the susceptibility of an entity and the persuasiveness of a context for a given model.
Let $\alphabet$ be an alphabet.
Consider a language model $\lm$ over $\alphabet$, i.e., $\lm$ is a distribution over the Kleene closure $\alphabet^*$.
Furthermore, we assume that $\lm$ was estimated from a corpus $\corpus \subset \alphabet^*$.
Let $\entities \subset \alphabet^*$ be the subset of strings that could correspond to string representations of entities.\footnote{In practice, an entity can have different verbalizations, e.g., \querytext{\entityexample{Sherlock}} and \querytext{\entityexample{Sherlock Holmes}.
}. 
And, whether a specific string refers to an entity may very well depend on the context in which it occurs.
Thus, our set $\entities$ is an approximation at best.}
Let $\queries = \{\query_n\}_{n=1}^N$ be a set of $N$ \defn{slotted query templates} of type $\query_n \colon \entities \rightarrow \Sigma^*$ that fill the slot with the argument to each $\query_n$. 
We can think of a query template $\query \in \queries$ as slotting an entity into a query, e.g., slotting the entity \entityexample{Slovenia} into the query template $\query(\entity) =$ \querytext{The capital of} $\entity$ \querytext{is} produces the string \querytext{The capital of \entityexample{Slovenia} is}.\looseness=-1

Now, consider three random variables.
First, let $\rvContext$ be a $\alphabet^*$-valued random variable that stands for a context.
Second, let $\rvAnswer$ be a $\alphabet^*$-valued random variable whose values are answers.
Let $\rvEntity$ be a $\entities$-valued random variable.
The pushforward $\query(\rvEntity)$, then, is a $\alphabet^*$-valued random variable over slottings of entities according to query $\query \in \queries$.
The random variables $\rvContext$, $\query(\rvEntity)$ and $\rvAnswer$ are jointly distributed according to the following probability distribution
\begin{equation}\label{eq:base-prob}
\!\!\!p(\rvContext = \context, \query(\rvEntity) = \query(\entity),\rvAnswer = \answer) \!\defpropto \lm(\context \query(\entity) \answer),
\end{equation}
where $\context\query(\entity)\answer \in \alphabet^*$ is string concatenation. 
To formalize persuasion and susceptibility, we 
make use of the joint distribution $p$ heavily in the proceeding subsections.\looseness=-1

\subsection{Persuasion Score}\label{sec:persuasion_score}
For each context $\context$ that is prepended to a query template with a given entity slotted in, $\query(\entity)$, we wish to assign a \defn{persuasion score} $\pscoresymb$ to represent how successful that context is at altering a model's answer distribution.
This score depends on the specific queried entity $\entity$, because contexts themselves are often entity-dependent.
Intuitively, a context's persuasion score should measure how much the probability distribution of possible answers changes, averaged across all possible answers.
More precisely, we define our persuasion score $\pscoresymb(\context, \query(\entity))$ as the half-pointwise mutual information (half-PMI) between the context $\context$ and the answer random variable $\rvAnswer$, conditioned on the fixed query about an entity:\looseness=-1
\begin{align}
\begin{split}
        \pscoresymb(&\context, \query(\entity))
    \defequals \MI(\rvContext = \context; \rvAnswer \mid \query(\rvEntity) = \query(\entity)) \label{eq:context_pmi}\\
    &=\sum_{\answer \in \Sigma^*}p(\answer \mid \context, \query(\entity))  \log \frac{p(\answer \mid \context, \query(\entity))}{p (\answer \mid \query(\entity))} \\
    &=\mathrm{KL}(p(\rvAnswer \mid \context, \query(\entity)) \mid\mid p(\rvAnswer \mid \query(\entity))),
 \end{split}
\end{align}
where $p(\rvAnswer \mid \context, \query(\entity))$ and $p(\rvAnswer \mid \query(\entity))$ can be derived by marginalizing and conditioning \cref{eq:base-prob}.
\looseness=-1

The persuasion score of a context can then be interpreted as the degree (in nats) to which the context was able to change the model's answer distribution when prepended to a query.
When the persuasion score is at its lower bound of 0 nats, it indicates the context is completely \emph{unpersuasive}, i.e., it did not change the model's answer distribution at all.
Contexts with higher persuasion scores
change the answer distribution more, which is consistent when viewed through the lens of KL-divergence.\footnote{We include details on further equivalencies of half-PMI and other concepts in information theory in \Cref{app:halfpmi}.}

\subsection{Susceptibility Score}\label{sec:susceptibility_score}
For a given query template applied to an entity $\query(\entity)$, we further wish to assign a susceptibility score $\stickyscoresymb$ to $\query(\entity)$ which represents how easy it is to change a model's answer distribution.
Intuitively, the susceptibility score should measure how much a model's answer distribution to a query changes when prompted with additional context, averaged across all possible contexts and answers.
More precisely, we define the susceptibility score $\stickyscoresymb(\query(\entity))$ as the mutual information between the context and answer random variables, conditioned on a fixed query about an entity:
\begin{align}
\begin{split}
 \stickyscoresymb(\query(\entity)) &\defequals \sum_{\context \in \Sigma^*} p(\context) 
 \pscoresymb(\context, \query(\entity)) \\
 &= \MI(\rvContext; \rvAnswer \mid \query(\rvEntity) = \query(\entity)),
\end{split}
\label{eq:sus_mi}
\end{align}
where $p(\context)$ is the marginal distribution over contexts. 
Equivalent to the difference in entropy $\entropy(\rvAnswer \mid \query(\rvEntity) = \query(\entity)) - \entropy(\rvAnswer \mid \rvContext, \query(\rvEntity) = \query(\entity)) \label{eq:sus_mi_entropy}$, the susceptibility score represents the reduction in answer distribution uncertainty (in nats) when a query is preceded by a context.
A high susceptibility score means the model is highly influenced by context for the query about that entity, with its upper bound of $\entropy(\rvAnswer)$ indicating that context fully determines the answer.
A low score indicates the model's response is robust to context, with its lower bound of $0$ indicating no influence of context on the answer distribution.
We can use susceptibility to answer the question ``How much does the model's answer depend on its entity bias?'' with an information-theoretically grounded, interpretable scale based on the model's full behavior.\looseness=-1

\subsection{Entity-Independent Persuasion Score}\label{sec:confusion_score}
Persuasion scores can be further marginalized over entities. 
Analogous to our definition of the susceptibility score, we define the \defn{entity-independent persuasion score} of a context as how much the log probability distribution of possible answers changes, averaged across all possible entities and answers. 
We describe this further in \Cref{app:confusion_score}. 

\section{Experiments}
We now provide empirical evidence to further validate our metrics, characterize model behavior using the susceptibility and persuasion scores, and investigate reasons behind differences in susceptibility scores for different entities.

\subsection{Setup}
\label{sec:exp_setup}
For each of the \numRelations different relations from the YAGO knowledge graph, we collect 100 entities\footnote{50 are real entities (e.g., \querytext{\entitytext{Adele}}) sampled from YAGO, and 50 are fake entities (e.g., \querytext{\entitytext{Udo König}}) of the same entity class (e.g., \emph{Person}) generated with GPT-4 \citep{openai_gpt4_2023}.} and construct 600 random contexts from relation-specific context templates such that each entity is mentioned in 6 contexts.
The 600 contexts are evenly distributed between \emph{assertive},\footnote{E.g., \contexttext{Definitely, the capital of \{entity\} is \{answer\}.}} \emph{base},\footnote{E.g., \contexttext{The capital of \{entity\} is \{answer\}.}} and \emph{negation}\footnote{E.g., \contexttext{The capital of \{entity\} is not \{answer\}.}} context types.
We construct four query forms of each relation: two \emph{closed} questions, i.e., yes--no, and two \emph{open} questions.

Using these samples, we compute persuasion scores for each context according to \Cref{eq:context_pmi} and susceptibility scores for each entity according to \Cref{eq:sus_mi} for each of the four query forms for six Pythia models of different sizes\footnote{70m, 410m, 1.4b, 2.8b, 6.9b, 12b (8-bit quantized)} trained on the deduplicated Pile \citep{wolf-transformers, dettmers2022llmint8, biderman2023pythia}.
Due to computational constraints, we approximate the model's answer distribution with the next-token distribution over the model's vocabulary.\footnote{While it would be more precise to estimate the answer distribution by repeatedly sampling many model outputs, this is very computationally expensive.}
The detailed setup can be found in \Cref{app:main_exp_setup}.

\subsection{Empirically Validating Our Metrics}
\label{sec:val_scores}

\subsubsection{Estimating Scores}
Because persuasion and susceptibility scores both involve a countably infinite sum, we opt for a stochastic approximation scheme.
Specifically, we construct a Monte Carlo estimator.
We sample from a narrower set of constructed contexts and approximate the answer distribution with the next token distribution over the model's vocabulary, as described in \Cref{sec:exp_setup}. 
We take a sample size of 600 contexts.
While the Monte Carlo approximation itself results in a consistent estimator, the additional approximations mean we do not have a guarantee on the quality of the approximation as a whole. 
In \Cref{fig:var_sus_and_p}, we exhibit the variance of our estimator across three random seeds, i.e., sampled sets of context.\looseness=-1

\subsubsection{Validating Persuasion Scores}
\label{sec:val_persuasion_scores}

\paragraph{Convergent Validity.}
According to existing measurement modeling methods \citep{loevinger_measurement_1957, validity_1987, jackman_measurement_2008, measurement_theory_2009, quinn_measurement_2010, wallach-measurement-2021}, observing a relationship between a new metric and existing ones would serve as additional evidence that the metric is meaningful.
To this end, we explore whether contexts with higher persuasion scores tend to more successfully convince the model to agree with the context.
Using the Pythia-6.9b-deduped model, we generate an answer for each prepended context to a query--entity pair from \Cref{sec:exp_setup} and use simple string matching to map the answer to whether it agrees with the context, the original answer, or neither.
We then apply a permutation test ($k=10000, \alpha=0.05$ with the Benjamini--Hochberg (BH) correction \citep{bh_correction}) for whether contexts that elicited context-concordant answers have higher persuasion scores than those that did not alter the model from the original answer for each of the 122 query topics.
Within a query topic, we run separate tests for open and closed queries, since the entropy of the answer distribution for closed queries tends to be much lower than the entropy for open queries.
We find that for $59\%$ of open queries, contexts that persuade the model to output the in-context answer have significantly higher persuasion scores than non-persuasive ones.
Curiously, this behavior holds for only $34\%$ of closed queries.
We discuss this surprising result more in \Cref{sec:discussion}.
We further show a summary of the mean persuasion scores for the different kinds of elicited answers in \Cref{fig:pscore_mr}.

\begin{figure}
    \centering
    \includegraphics[width=\columnwidth]{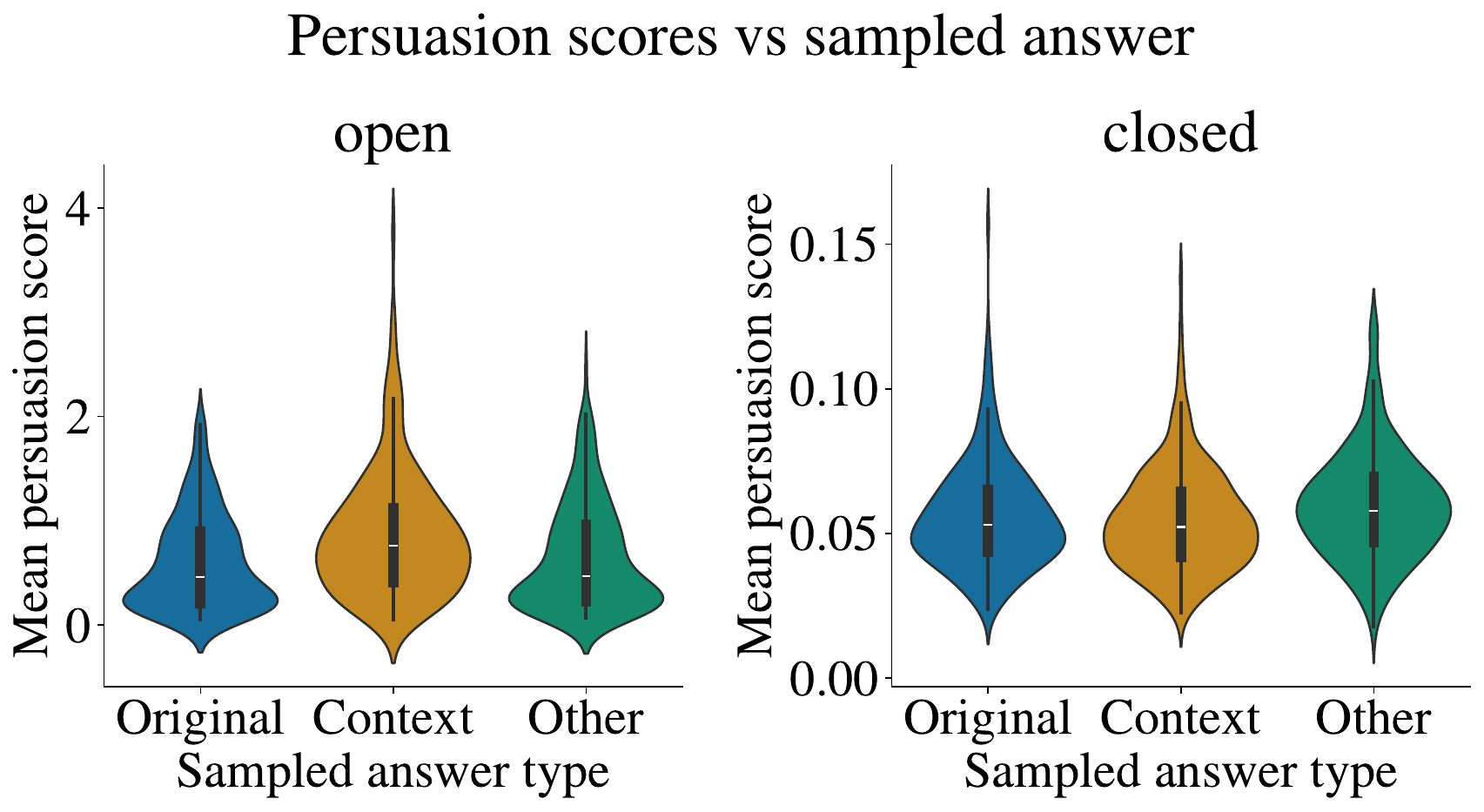}
    \caption{The \emph{x}-axis represents bins for whether the model's answer agreed with its prior, the context, or neither.
    For open (left) queries, the persuasion scores of contexts that persuaded the model to output an answer matching the context (\contexttext{Context}) are higher than those of contexts that did not (\entitytext{Original}, \othertext{Other}).}
    \label{fig:pscore_mr}
\end{figure}

\paragraph{Construct Reliability.}
To show construct reliability, we provide evidence that persuasion scores do not strongly vary for inputs that differ only in phrasing, i.e., \emph{query form}.
For example, the persuasion score of \contexttext{The capital of Slovenia is Oz.} should be similar when prepended to either 
\querytext{Is \entityexample{Slovenia}'s capital Ljubljana?} or \querytext{Is Ljubljana the capital of \entityexample{Slovenia}?}.
We compute persuasion scores using the setup from \Cref{sec:exp_setup} with two query forms for both closed and open queries, keeping contexts that appeared for the same query and entity for all seeds.
We then compute the variance across the query forms to test for reliability.
\Cref{fig:var_sus_and_p} shows strong evidence that the metric is reliable.
The variance is very low across different closed query forms, as expected.
While open query forms have higher variance, this is not unexpected because the question-answering query form, e.g., \querytext{Q: What is the capital of \entityexample{Slovenia}?\textbackslash nA:} is more specific than the sentence-completion form, e.g., \querytext{The capital of \entityexample{Slovenia} is}, which has a broader set of plausible answers.\looseness=-1

\subsubsection{Validating Susceptibility Scores}
\label{sec:val_sus_scores}

\paragraph{Convergent Validity.}

We compare susceptibility scores to per-entity memorization ratio (MR)\footnote{Adapted from \citet{longpre_entity-based_2022} to apply on a \emph{per-entity} basis and describe it in more detail in \Cref{app:mr}.} as further evidence for the meaningfulness of our metric, using the setup from \Cref{sec:exp_setup} for the Pythia-6.9b-deduped model.
\Cref{fig:sus_mr} shows a decreasing upper-bounding relationship between susceptibility scores and MR for both open queries and closed queries.
While not a straightforward correlation, this pattern supports the convergent validity of susceptibility scores and can be explained by the scores' nature of measuring a difference in entropy.
That is, if the model's answer was often changed (low MR), this could correspond to a wide range of entropy change (low or high susceptibility), depending on the model's confidence without the context. If the model's answer was mostly unchanged (high MR), the entropy likely remained similar (low susceptibility).
We discuss these results further in \Cref{app:convergent}.\looseness=-1

\begin{figure}[t]
    \centering    
    \includegraphics[width=\columnwidth]{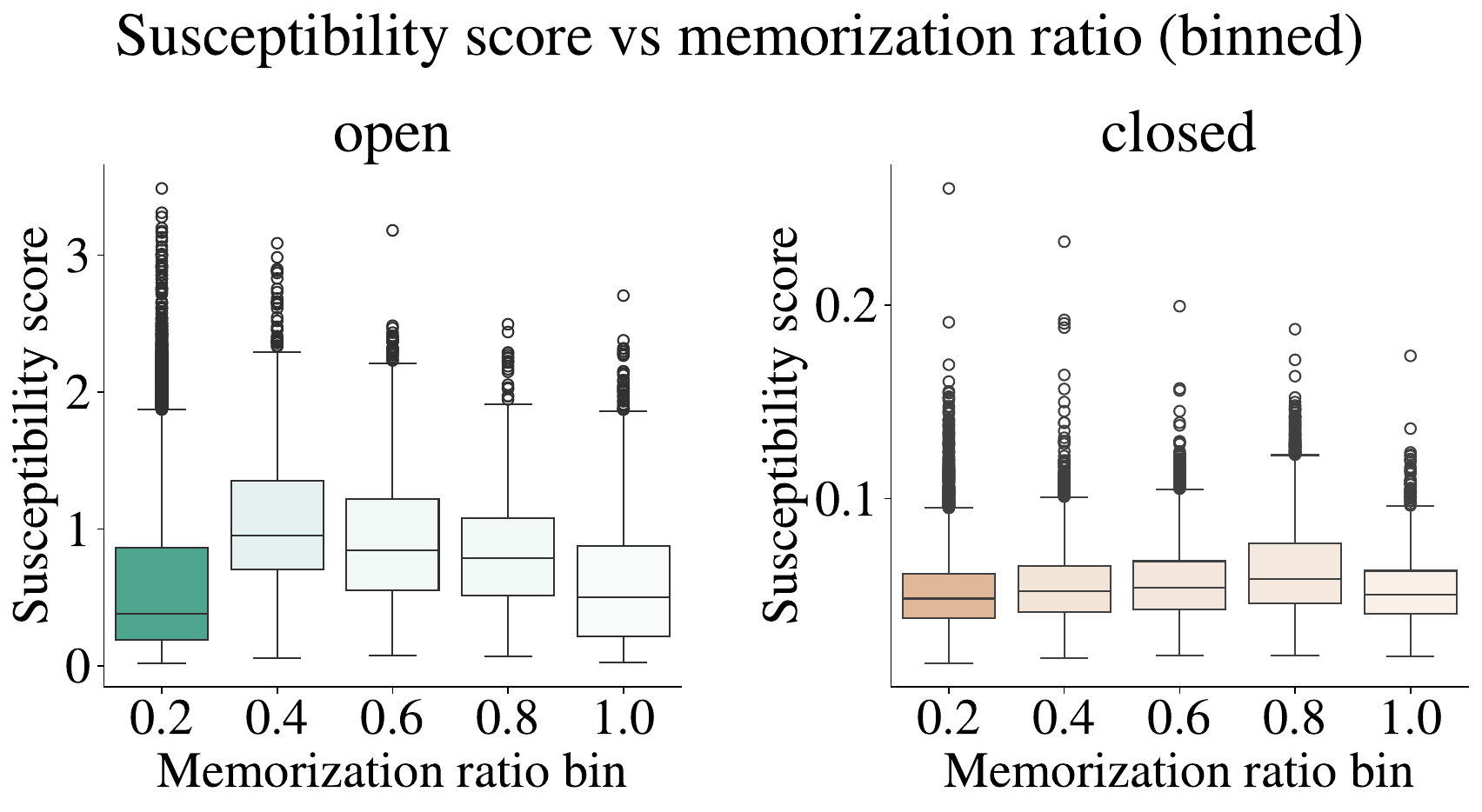}
    \caption{Susceptibility score  ($y$-axis) against the MR divided into 5 bins between 0 and 1 ($x$-axis) for all entities and queries. 
    The opacity represents the proportion of points in each bin.
    For both \opentext{open queries} (\opentext{$\bullet$}) and \closedtext{closed queries} (\closedtext{$\bullet$}), we see a decreasing upper bound between the MR and susceptibility score.
    While the quartiles of the open queries generally decrease (except for the lowest bin), the opposite occurs for closed queries. 
    \looseness=-1}
    \label{fig:sus_mr}
\end{figure}

\paragraph{Construct Reliability.}

Following the setup used in \Cref{sec:val_persuasion_scores}, we consider the same two query forms for both open and closed questions to test for reliability (low variance) for susceptibility scores.
\Cref{fig:var_sus_and_p} shows strong evidence for the reliability of susceptibility scores.
Similarly to the persuasion scores, the variance for both open and closed queries is very low across closed query forms; open query forms have higher variance because they have meaningfully different possible answers, as discussed in \Cref{sec:val_persuasion_scores}.\looseness=-1

\begin{figure}[t]
    \centering
    \includegraphics[width=\columnwidth]{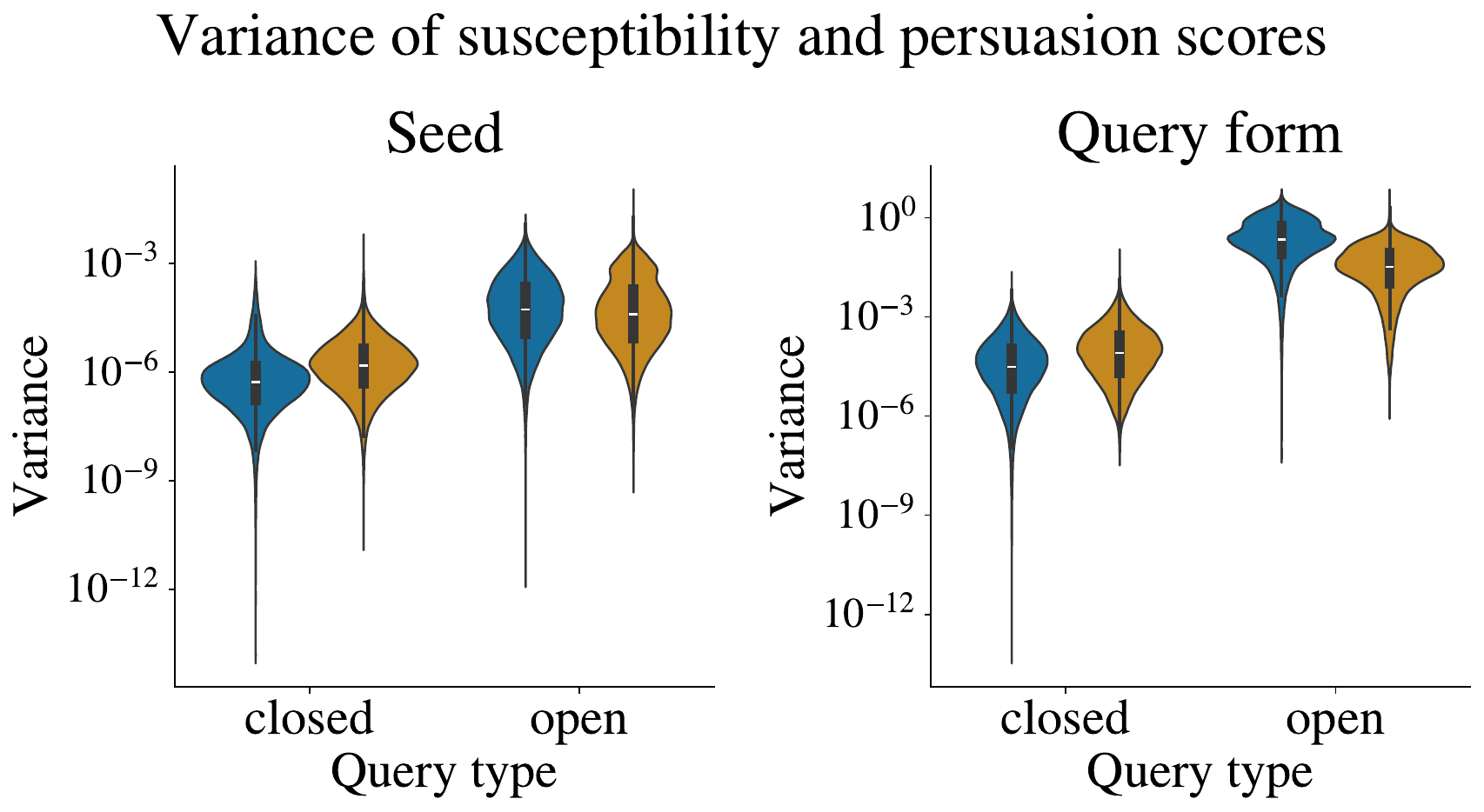}
    \caption{Summarizing across all \numRelations queries, we display the variance of \entitytext{susceptibility scores} (blue) and \contexttext{persuasion scores} (orange), across three random seeds (left) and across two query forms (right), and stratified for both closed and open queries ($x$-axis).
    The variance is very low across random seeds for both query types, and, for closed queries, across the specific query form.
    Variance is high for the different open query forms.}
    \label{fig:var_sus_and_p}
    \vspace{-0.5cm}
\end{figure}

\subsection{What Makes a Context Persuasive?}
\label{sec:persuasion_pred_val}

To better characterize model behavior with persuasion scores, we explore several tests for qualities that might distinguish more persuasive contexts from less persuasive ones: relevance, assertiveness, and negation.\looseness=-1

\begin{figure*}[t]
    {\centering
    Permutation test results across models and comparisons\par\medskip}
    \begin{subfigure}[b]{0.5\textwidth}
        \includegraphics[width=\textwidth]{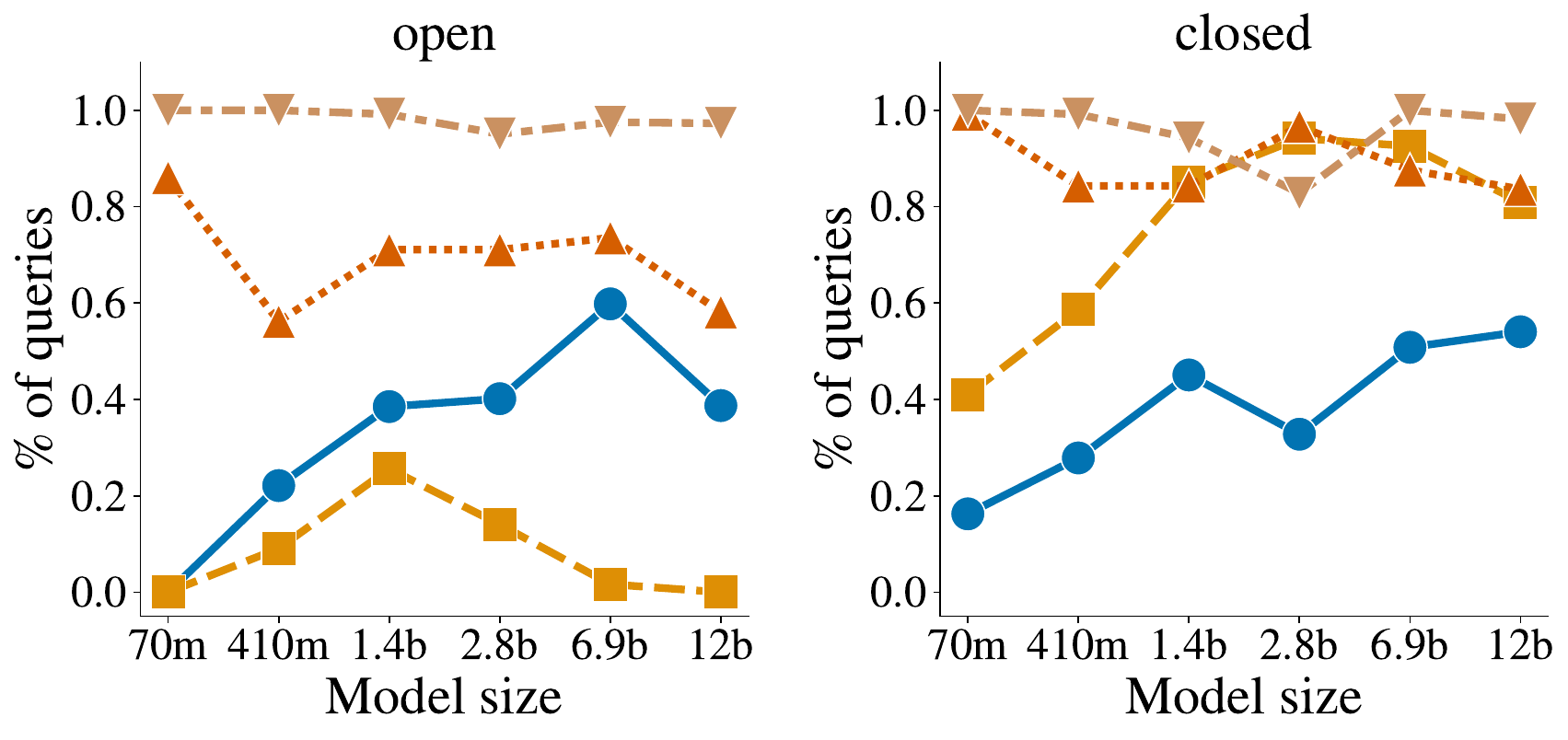}
        \caption{\% of queries with significant result}
        \label{fig:models_sig_prop}
    \end{subfigure}
    \hfill
    \begin{subfigure}[b]{0.5\textwidth}
    \includegraphics[width=\textwidth]{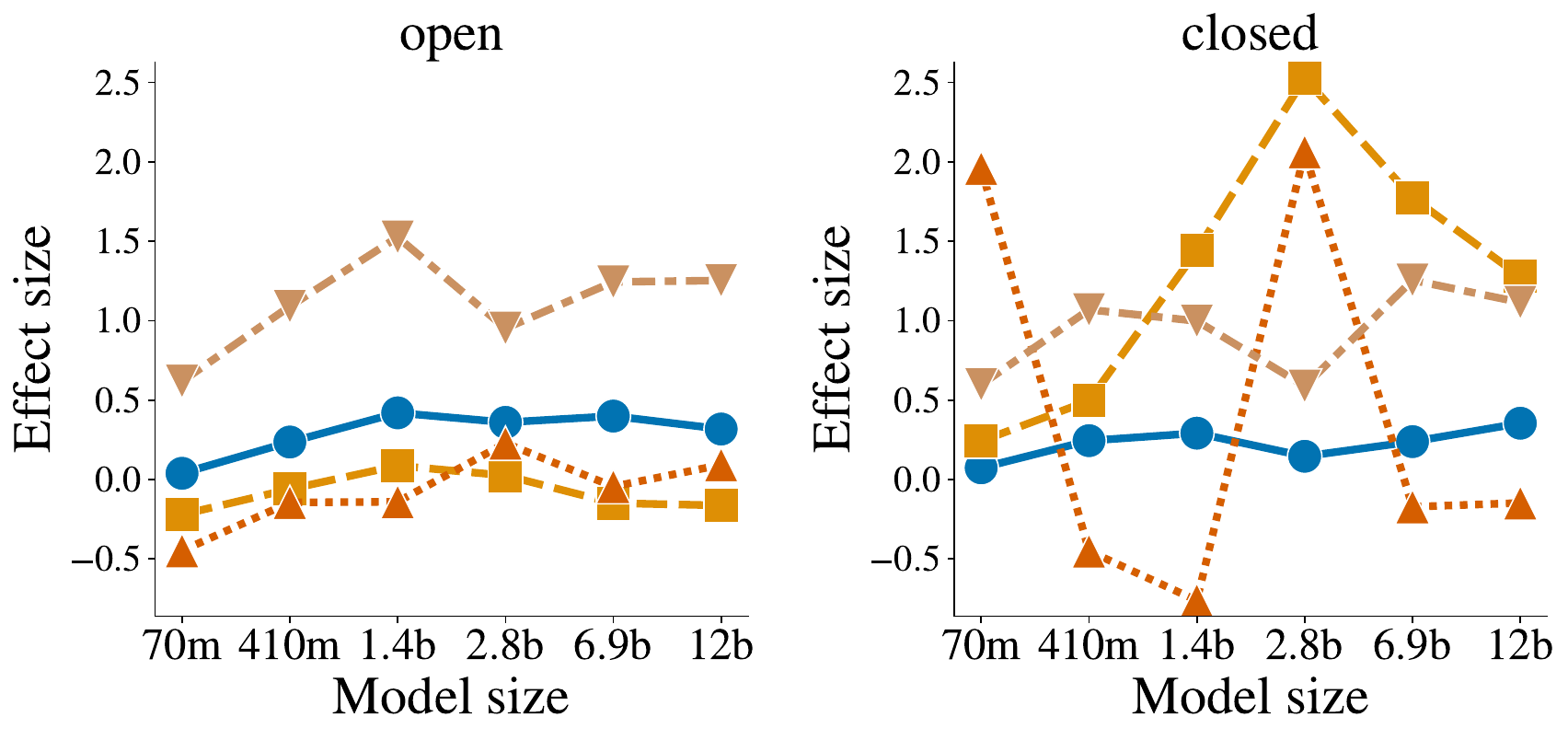}
        \caption{Mean effect size}
        \label{fig:models_effect_sz}
    \end{subfigure}
    \caption{The plots in \Cref{fig:models_sig_prop} indicate the proportion of queries for which ($\color{MyBronze}{\blacktriangledown}$) relevant contexts are significantly more persuasive than irrelevant contexts, ($\color{MyBlue}{\bullet}$) unfamiliar entities are significantly more susceptible than familiar entities, ({\tiny$\color{MyOrange}{\blacksquare}$}) assertive contexts are significantly more persuasive than base contexts, and ($\color{MyRed}{\blacktriangle}$) negation contexts are significantly more persuasive than base contexts. 
    We further provide the average effect size over queries of those comparisons in \Cref{fig:models_effect_sz}.
    We highlight specific findings in \Cref{sec:persuasion_pred_val} and \Cref{sec:sus_pred_validity_real_vs_fake}.
    }
    \label{fig:across_models}
    \vspace{-10pt}
\end{figure*}

\subsubsection{Relevance}
\label{sec:p_relevancy}
\paragraph{Experiment Setup.}
We use the setup described in \Cref{sec:exp_setup}.
For the relevance test, we consider a \emph{relevant} context to be one that mentions the queried entity, and an \emph{irrelevant} context as one that does not.
We hypothesize that \emph{relevant} contexts should be more persuasive than \emph{irrelevant} ones.
For each of the \numRelations queries, we find the mean persuasion score of the relevant contexts and irrelevant contexts and use a permutation test to determine whether the mean persuasion score for relevant contexts (across entities) is higher than the mean persuasion score for irrelevant contexts, using a significance level of $\alpha=0.05$ with the BH correction.

\paragraph{Results.}
As seen in \Cref{fig:across_models} ($\color{MyBronze}{\blacktriangledown}$), 
across most model sizes and relations, relevant contexts are significantly more persuasive than irrelevant contexts.
Specifically, depending on the model, 95--100\% of open queries and 83--100\% of closed queries showed a significant result.
We also see a trend in which, as model size increases, so too does the degree to which relevant contexts are more persuasive, as measured by the mean effect size.
We summarize the significance test results and effect sizes for all models and queries in \Cref{app:p_score_in_depth}.

\subsubsection{Assertiveness}
\label{sec:p_assertiveness}
\paragraph{Experiment Setup.}
Using a similar setup to \Cref{sec:p_relevancy}, we explore whether \emph{assertive} contexts are more persuasive than \emph{base} ones by testing whether the mean persuasion score of the former group is greater than that of the latter for each query.

\paragraph{Results.}
\Cref{fig:across_models} ({\tiny$\color{MyOrange}{\blacksquare}$}) shows that consistently across model sizes, assertive contexts tend to be more persuasive than base contexts for closed queries but not for open queries.
Moreover, assertive contexts are most persuasive for medium-sized models such as 1.4b and 2.8b.
We hypothesize that smaller models may be less persuaded by assertive contexts because they may be worse at integrating context into their answers. 
Larger models may be less persuaded because of their stronger prior knowledge/lower susceptibility to context, whether assertive or not.
We highlight the significance test results and effect sizes for all queries for the Pythia-6.9b-deduped model in \Cref{fig:p_score_assertive} and other models in \Cref{app:p_score_in_depth}.\looseness=-1

\subsubsection{Negation}
\label{sec:p_negation}

\paragraph{Experiment Setup.}
We use the same setup as in \Cref{sec:p_assertiveness} and explore whether \emph{negation} contexts differ in persuasiveness from the \emph{base} ones, using a two-tailed permutation test for each query.

\paragraph{Results.}
The permutation tests suggest evidence for a significant difference in 88\% of the closed queries and 74\% of the open queries for the Pythia-6.9b-deduped model.
However, there is no consistent directional pattern; from \Cref{fig:p_score_negation}, we see that negations are significantly more persuasive for some queries while significantly less persuasive for others. 
\Cref{app:p_score_in_depth} shows a similar pattern for other models.
\Cref{fig:across_models} ($\color{MyRed}{\blacktriangle}$) also shows that the smallest model is the most sensitive to being persuaded differently by negation vs base contexts. 
For closed queries, this may be due to potential spurious correlations or token biases between, i.e., seeing the word \emph{not} and the model's probability of outputting \emph{No}.\looseness=-1

\begin{figure}
    \centering
    \includegraphics[width=\columnwidth]{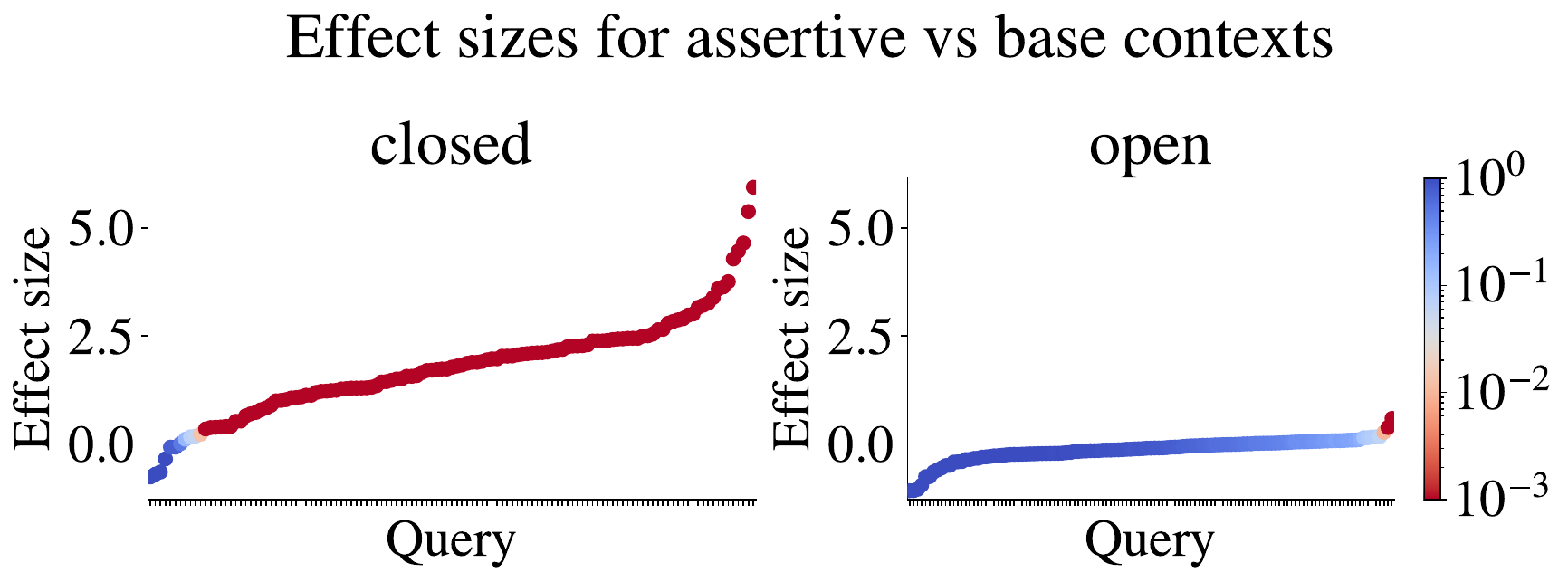}
    \vspace{-0.7cm}
    \caption{These plots show, for the Pythia-6.9b-deduped model, the effect size between base and assertive contexts ($y$-axis) and p-values ({\color{MyRed}{\emph{red}}} is significant, {\color{MyInsignificantBlue}{\emph{blue}}} is insignificant) of the null hypothesis that persuasion scores of assertive contexts are not greater than those of base contexts, for each of the \numRelations queries ($x$-axis). 
    The evidence suggests that assertive contexts are significantly more persuasive than assertive contexts for most closed queries, but few open queries.\looseness=-1
    }
    \label{fig:p_score_assertive}
 \vspace{-5pt}
\end{figure}

\subsubsection{Comparing Context Qualities}
For all models on open queries, the relevance of a context has the greatest effect on persuasion score compared to assertiveness or negation, as measured by 
effect size. 
This is consistent with our intuition that the model should be more sensitive to whether the queried entity is mentioned in the context than to other context features.
However, we can see in \Cref{fig:across_models} that for medium-sized models on closed queries, the other context comparisons have stronger effect sizes, e.g., assertiveness and negation for the 2.8b model.
This could be explained by potential spurious correlations/token biases associating, i.e., the word \emph{not} with the model outputting \emph{No} or the word \emph{definitely} with outputting \emph{Yes}.

\begin{figure}
    \centering
    \includegraphics[width=\columnwidth]{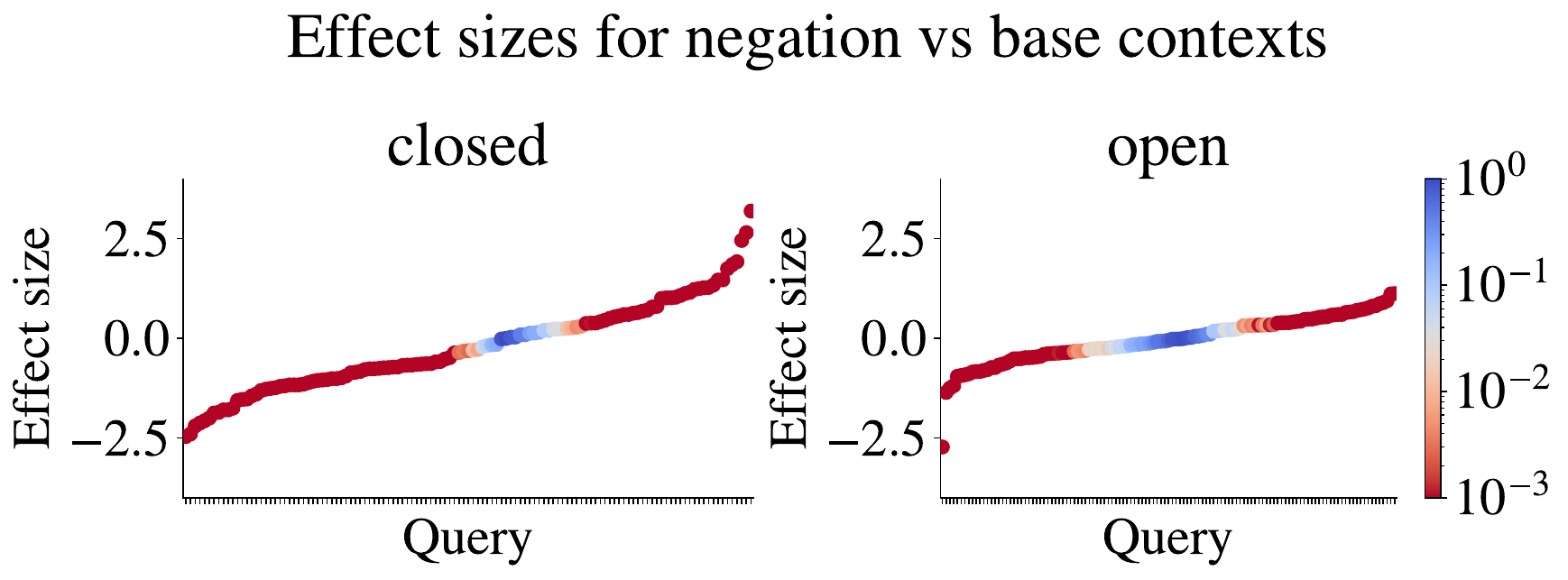}
    \vspace{-0.7cm}
    \caption{These plots show, for the Pythia-6.9b-deduped model, the effect size ($y$-axis) and p-values ({\color{MyRed}{\emph{red}}} is significant, {\color{MyInsignificantBlue}{\emph{blue}}} is insignificant) of the null hypothesis that persuasion scores of negation contexts are the same as those of base contexts, for each of the \numRelations queries ($x$-axis).
    While many queries are significant, some queries exhibit a significantly positive effect while others exhibit a significantly negative one.\looseness=-1
    }
    \label{fig:p_score_negation}
\end{figure}

\subsection{What Makes an Entity Susceptible?}
\label{sec:sus_pred_validity_real_vs_fake}

\subsubsection{Familiar vs Unfamiliar Entities}
We hypothesize that the model should have lower susceptibility scores for entities encountered during pretraining compared to unfamiliar, fake entities.
\paragraph{Experiment Setup.}
We use the setup described in \Cref{sec:exp_setup} to compute susceptibility scores for known and unknown entities.
We use a permutation test with $\alpha=0.05$ and the BH correction to test whether the known real entities are less susceptible than the unknown fake entities.\footnote{We filter out fake entities that appear in the training data to better represent unknown entities in our analysis.}
The detailed setup can be found in \Cref{app:main_exp_setup}.
\paragraph{Results.}

\begin{figure}
    \centering
    \includegraphics[width=\columnwidth]{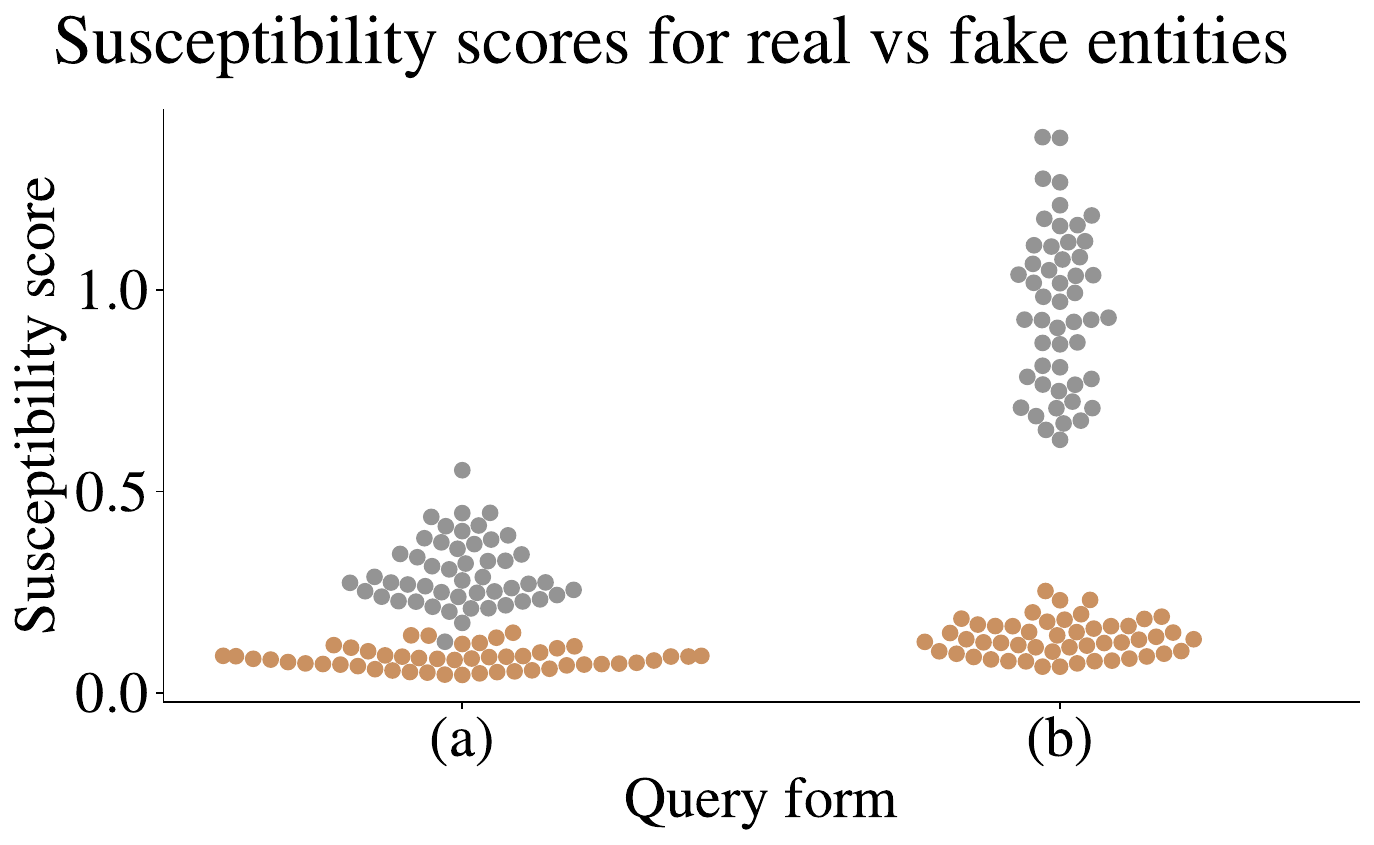}
    \vspace{-0.7cm}
    \caption{The \querytext{officialLanguage} query topic exhibits a particularly strong effect in which \realtext{real entities} ($\color{MyBronze}{\bullet}$) have lower susceptibility scores than \faketext{fake entities} ($\color{MySilver}{\bullet}$) for open questions.
    This plot shows susceptibility scores for 100 real and fake entities, for two different query forms: (a) \querytext{Q: What is the official language of \entityexample{\{entity\}}?\textbackslash nA:}, and (b) \querytext{The official language of \entityexample{\{entity\}} is}.
    }
    \label{fig:real_vs_fake_example}

\end{figure}

For the Pythia-6.9b-deduped model, we find that with $\frac{73}{\numRelations}$ queries (open questions) and $\frac{61}{\numRelations}$ queries (closed questions), familiar entities have significantly lower susceptibility scores than unfamiliar fake entities.
We conjecture that for the remaining queries, the model may not have strong prior biases about the sampled entities for these queries.
Indeed, further analysis (\Cref{app:effect_vs_frequency}) finds some evidence supporting the hypothesis that queries with smaller effect sizes or less significant $p$-values feature less familiar entities, as we find a small correlation between effect size and entity frequency in the training set. 
\Cref{fig:real_vs_fake_example} shows the distribution of susceptibility scores for real and fake entities for an example query with a particularly strong effect size.\looseness=-1

Generally, as model size increases, so too does the significance and effect size of unfamiliar entities being more susceptible than familiar entities, as seen by the generally increasing blue lines ($\color{MyBlue}{\bullet}$) in all four plots in \Cref{fig:across_models}.
This trend is consistent with our expectation that bigger models have stronger prior knowledge of entities and are therefore less susceptible to familiar entities.
The smallest model (70m) does not have a significant difference in susceptibility between familiar and unfamiliar entities, which could indicate it is too small to have strong prior knowledge.\looseness=-1

\subsubsection{Degrees of Familiarity}
\paragraph{Training Data Frequency.}
Since language models are parameterized with knowledge from their training corpora, we hypothesize that the model is \emph{less susceptible} for entities with which it is \emph{more familiar}, i.e., more frequently occurring in the training data.
We investigate this relationship between the Pythia models' behavior and frequency statistics in the Pile dataset on which they were trained \cite{pile}.
To capture the model's familiarity with an entity--answer relation, we count the number of co-occurrences between the entity and its corresponding answer within a 50-word window.
We compare the susceptibility score to this co-occurrence frequency and find a significant correlation (Spearman $\rho=-0.23$) for the Pythia-6.9b-deduped model.
We see in \Cref{fig:sus_to_freq_max_dist_50} that as the training data frequency increases, the susceptibility scores' upper bound decreases.
This trend is shared across all model sizes (see \Cref{app:suscept_freq_deg}). \looseness=-1

\begin{figure}
    \centering
    \includegraphics[width=\columnwidth]{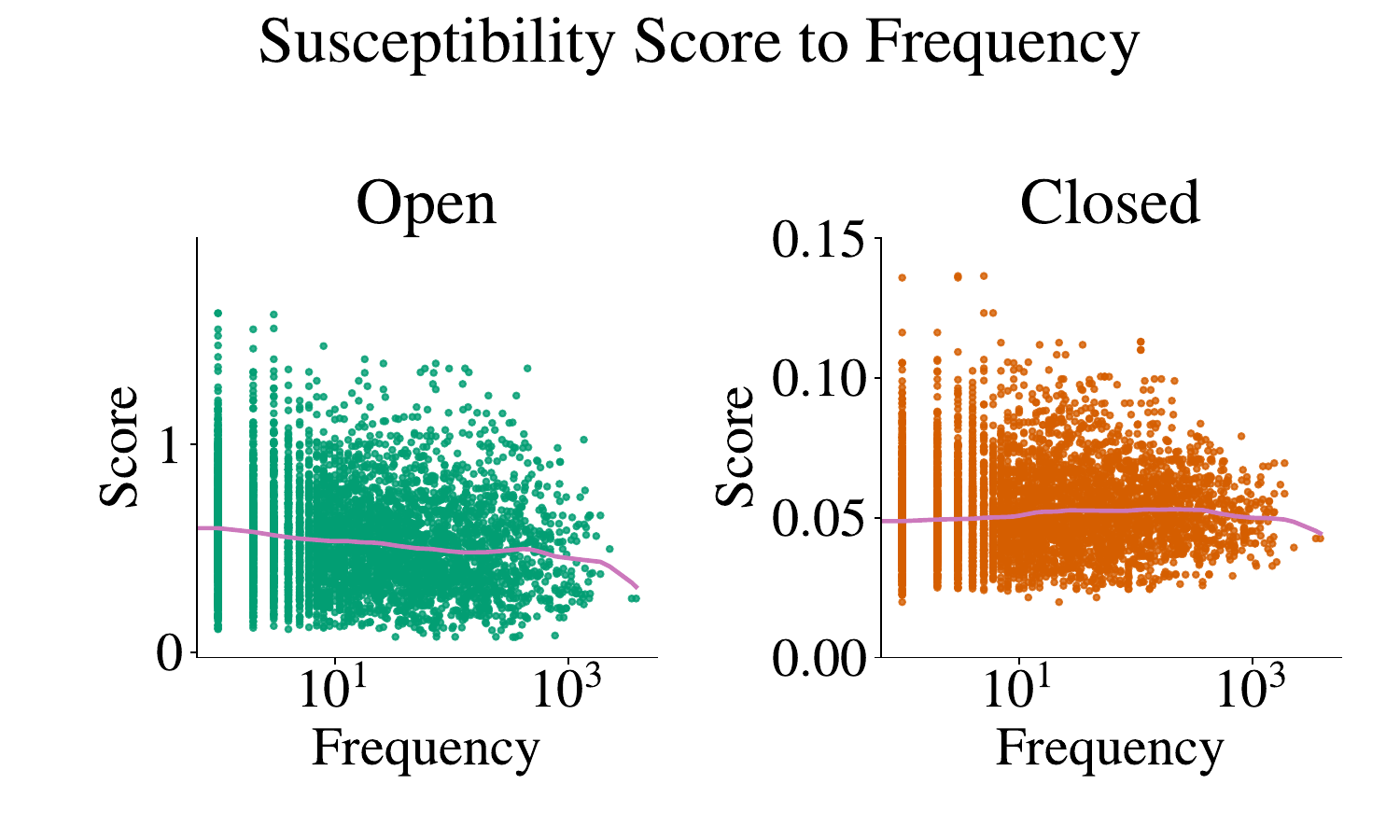}
    \vspace{-0.85cm}
    \caption{For both \opentext{open} (\opentext{$\bullet$}) 
    and \closedtext{closed} (\closedtext{$\bullet$}) queries, the upper bound of the mean susceptibility scores for entity-answer pairs decreases as co-occurrence frequency in the Pile increases (Pythia-6.9b-deduped).}
    \label{fig:sus_to_freq_max_dist_50}
\end{figure}

\paragraph{Entity Degree in a Knowledge Graph.}
\label{sec:sus_pred_validity_yago}
Training data can be noisy, difficult, and expensive to search through for entity co-occurrences and more complex frequency statistics.
Knowledge graphs offer a more precise alternative as they represent entities and relations extracted from common corpora like Wikipedia in a structured way.
For example, within the pretraining data, it is very difficult to identify the number of different answers with which an entity will co-occur within the context of a specific relation. 
However, with a knowledge graph, we can easily identify the exact number of objects with which an entity might share a given relation.
Thus, we explore the relationship between the relation-dependent degree of an entity in a knowledge graph and the susceptibility score.
Like with training data frequency, we find that comparing against this degree yields a similar-looking plot with a decreasing upper bound between susceptibility scores and the degree. 
Such a trend suggests that unfamiliar entities have the potential to be highly susceptible, while very familiar entities tend to have lower susceptibility.
In \Cref{app:suscept_freq_deg}, we show this trend is shared across all model sizes and provide our methodology in more detail.\looseness=-1

\looseness=-1

\section{Applications}
\label{sec:apps}
We examine how knowing susceptibility scores can be useful for analyzing model behavior in two different applications.
Here, we highlight one key finding per application; however, we emphasize that future work can conduct a more detailed analysis of these and other applications.

\paragraph{Social Sciences Measurement.}
Social scientists use large language models (LLMs) for annotating data and descriptive data analysis, yet such use may inadvertently incorporate entity biases and skew model behavior \citep{ziems_can_2022, gilardi_chatgpt_2023, zhang_sentiment_2023, ohagan_measurement_2023}.
To better contextualize model annotations for a case study on friend--enemy stance detection \citep{choi-etal-2016-document, stoehr-etal-2023-ordinal}, we aim to understand how susceptibility scores may differ between friend and enemy-based entity pairs for the query \querytext{The relationship between \entityexample{\{entity1\}} and \entityexample{\{entity2\}} is}, e.g., are famous friend-based relationships more susceptible than enemy-based relationships?
From \Cref{fig:friend--enemy}, we see that with the Pythia-6.9b-deduped model for two specific query forms, enemy duos are less susceptible than friend duos, which can inform social scientists that a model's annotation for friend pairs may be more easily influenced by the context than enemy pairs.
We provide more details in \cref{app:apps}. 
\looseness=-1

\begin{figure}
    \centering
    \includegraphics[width=\columnwidth]{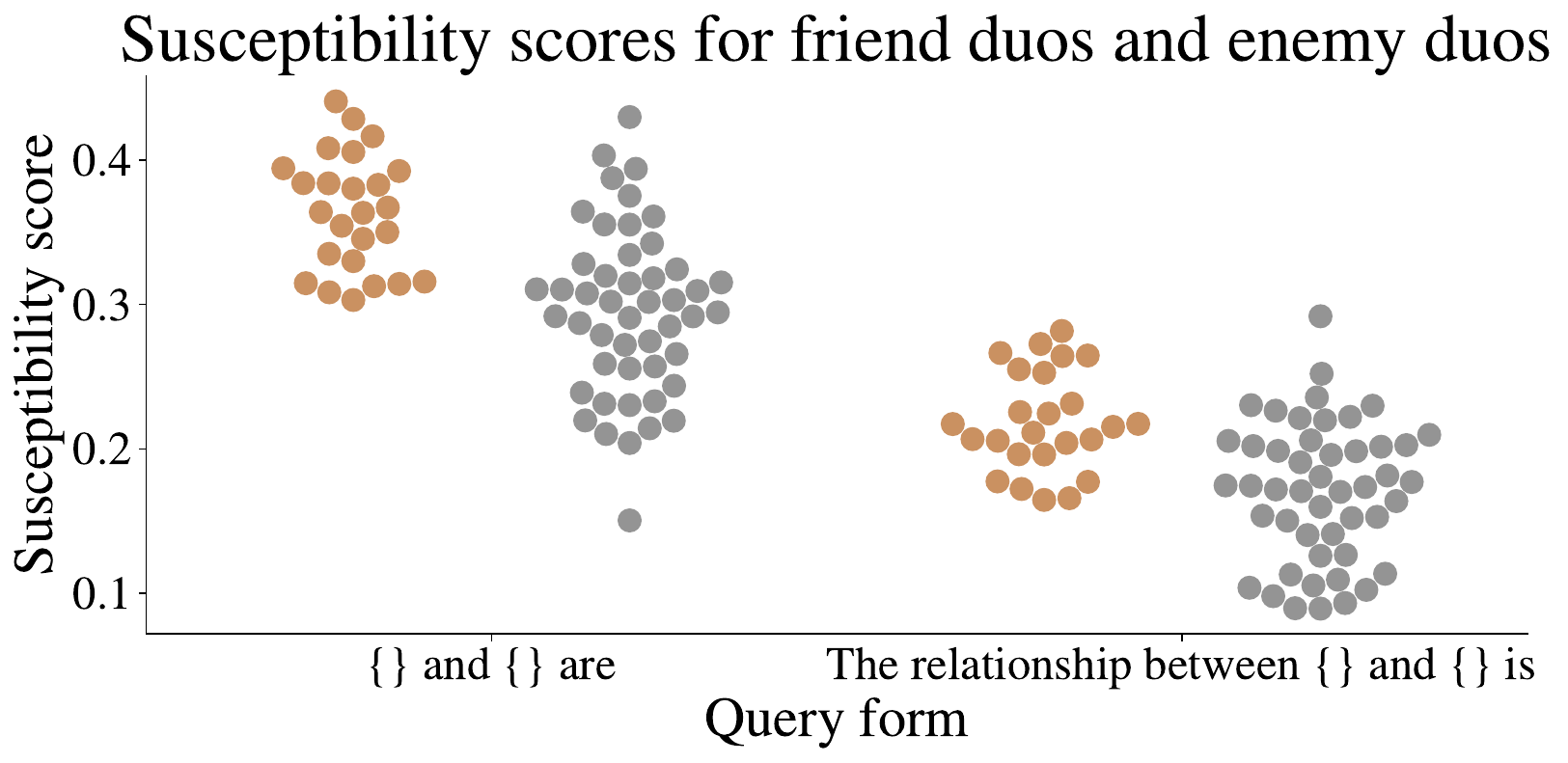}
    \caption{Susceptibility scores for different entity-pairs which are either friends or enemies. \faketext{Enemy duos} (\faketext{$\bullet$}) appear to have lower susceptibility than \realtext{friend duos} (\realtext{$\bullet$}). \looseness=-1}
    \label{fig:friend--enemy}
\end{figure}

\paragraph{Exploring Gender Bias.}
\label{sec:gender_bias}
Since higher susceptibility scores indicate weaker induced biases for entities, we conjecture that this can relate to being underrepresented in the training data. 
Based on this, we consider how the susceptibility score can be used to study gender bias in LLMs. 
Using GPT-4, we collect stereotypically biased contexts, gendered names, and neutral queries and run several experiments to identify gender discrepancies in susceptibility scores; see \Cref{app:apps} for full details. 
We highlight a result where masculine names have a higher susceptibility than feminine names when swapping the genders in the stereotypical contexts, as seen in \cref{fig:gender-bias}.
This could indicate that the model is more surprised to see contexts claiming men follow feminine stereotypes, and therefore could suggest less representation of feminine stereotypes in the training data than masculine ones.
\looseness=-1

\begin{figure}
    \centering    \includegraphics[width=0.9\columnwidth]{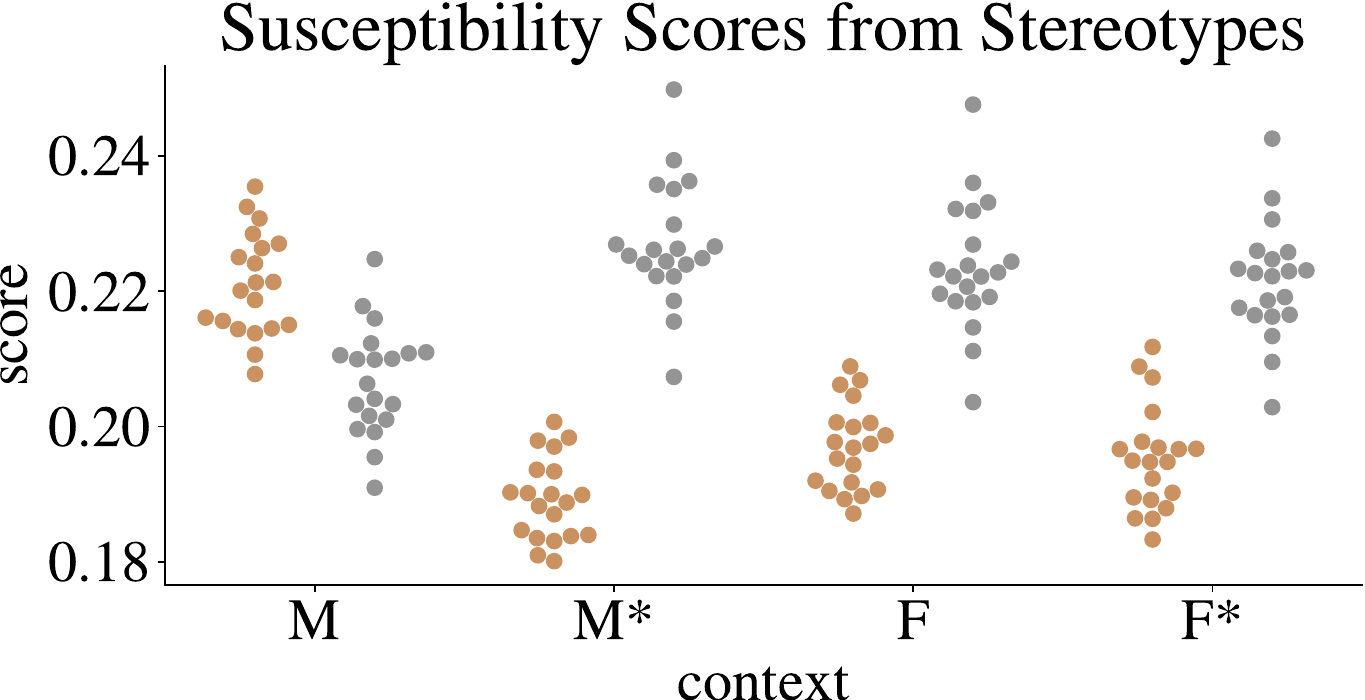}
    \caption{Susceptibility scores for the gendered names (\faketext{masculine} (\faketext{$\bullet$}), \realtext{feminine} (\realtext{$\bullet$})) over the stereotypical contexts. M and F are the original stereotypes, and M* and F* correspond to the swapped genders.} 
    \label{fig:gender-bias}
\end{figure}

\section{Discussion and Conclusion}
\label{sec:discussion}
We have made the case that the \emph{persuasion score} and \emph{susceptibility score} are both \emph{valid} and \emph{reliable} in measuring their respective constructs, following a well-established measurement modeling framework. 
Throughout our experiments, we find a common theme in the results: there is a strong, negative, possibly linear relationship between the upper bound of the susceptibility score and (a) the entity's memorization ratio (\Cref{fig:sus_mr}), (b) the log-co-occurrence frequency in the training data (\Cref{fig:sus_to_freq_max_dist_50}), and (c) the log-relation-dependent degree in a knowledge graph.
That is, as each of those three values increases, we see a clear pattern indicating that the highest susceptibility scores tend to decrease. 
This is consistent with our hypothesis that the induced bias of an entity increases for a model as the model's expected familiarity with the entity increases.
Furthermore, we find a difference in behavior between scores for \emph{open} and \emph{closed} queries; while in many experiments, we see similar patterns between the two, susceptibility and persuasion appear to have a stronger relationship with memorization ratio for open queries, while assertive contexts appear to be significantly more persuasive than base contexts primarily for closed queries.
This difference is a surprising phenomenon that warrants future study; we hypothesize closed queries may behave this way due to common token biases (with its output space of \emph{Yes} and \emph{No}) or the influence of mentioning both an entity and an answer in the query.
Finally, while we applied these metrics to analyze model behavior in two case studies, in future work we aim to apply them to unearth new perspectives in other context and prior knowledge-dependent problems such as retrieval-augmented generation, model editing and control, and few-shot learning.\looseness=-1

\section*{Limitations}
We face some technical limitations in executing the empirical aspects of this work.
First, while \Cref{sec:formalization} defines the output space of $\rvAnswer$ as the set of all possible outputs $\Sigma^*$, in practice, it is computationally expensive to estimate that probability distribution. 
Instead, we look only at the model's probability distribution of the next token, which could be a noisy signal, especially in cases where the answer suggested by a context and the answer suggested from prior knowledge share the same first token.
Second, it is difficult to go through the whole Pile to count answer--entity co-occurrences without noise.
Third, the scores depend on the sampled contexts, which may not be representative of all applications.

\section*{Ethics Statement}
As LLM capabilities grow more advanced and their usage proliferates throughout the real world, we acknowledge that their development can exacerbate risks to people, especially those historically underrepresented or misrepresented to these models.
Our work aims to make model behavior more transparent by providing a new tool to analyze the interaction between context and prior knowledge in LMs, which is especially important as people interact with them in chat, question-answering, and other prompt-based settings.
We foresee no particular ethical concerns and hope this paper contributes to developing tools that can identify and mitigate ethical concerns in the future.

\section*{Acknowledgements}
We thank Alex Warstadt, Tiago Pimentel, Anej Svete, Alexandra Butoi, Benjamin Dayan, and Leo Du for helpful comments, discussion, and feedback. Niklas Stoehr acknowledges funding through the Swiss Data Science Center (SDSC) Fellowship. Vésteinn Snæbjarnarson is supported by the Pioneer Centre for AI, DNRF grant number P1.

\bibliography{references/anthology, references/custom, references/12-12-23-context-bias, references/2-7-24}
\bibliographystyle{acl_natbib}

\appendix
\newpage
\onecolumn
\section{A Primer on Half-pointwise Mutual Information}
\label{app:halfpmi}
Half-pointwise mutual information (HPMI) is a non-standard concept.
There is, for example, no mention of it in standard references on information theory \citep{Cover2006}.
In this brief primer, we give various properties of HPMI and show how it relates to other concepts in information theory.
Given random variable $X$ over the discrete space $\calX$, random variable $Y$ over the discrete space $\calY$, and $x \in \calX$, the half-PMI of $X=x, Y$ is defined as:
\begin{equation}
\hpmi(X=x; Y) \defequals \sum_{y \in  \calY}{p(y \mid x) \log \frac{p(y \mid x)}{p(y)} }
\end{equation}
such that
\begin{subequations}
\begin{align}
\expect_{x \sim X}\left[\hpmi(X=x; Y)\right] &= \expect_{x \sim X}\left[\sum_{y \in  \calY}{p(y \mid x) \log \frac{p(y \mid x)}{p(y)}}\right] \\
    &= \sum_{x \in  \calX}\sum_{y \in  \calY}{p(x)p(y \mid x) \log \frac{p(y \mid x)}{p(y)}} \\
    &= \sum_{x \in  \calX}\sum_{y \in  \calY}{p(x, y) \log \frac{p(x, y)}{p(x)p(y)}} \\
    &= \mathrm{MI}(X; Y)
\end{align}
\end{subequations}

We now state and prove three Propositions about HPMI.
\begin{prop}
\label{prop:hx_minus_hxgiveny}
Given random variable $X$ over the discrete space $\calX$, random variable $Y$ over the discrete space $\calY$, and $x \in \calX$, then:
\begin{equation}
    \hpmi(X=x; Y) = \ent(X=x) - \ent(X=x \mid Y)
\end{equation}
\end{prop}
\begin{proof}
\begin{subequations}
\begin{align}
\hpmi(X=x; Y) 
&\defequals \sum_{y \in  \calY}{p(y \mid x) \log \frac{p(y \mid x)}{p(y)} } \\
&= \sum_{y \in \calY} p(y \mid x) \log \frac{p(x, y)}{p(x) p(y)} \\
&= \sum_{y \in \calY} p(y \mid x) \log \frac{p(x \mid y)}{p(x)} \\
&= - \sum_{y \in \calY} p(y \mid x)  \log p(x)  + \sum_{y \in \calY} p(y \mid x) \log p(x \mid y) \\
&=- \log p(x)  + \sum_{y \in \calY} p(y \mid x) \log p(x \mid y) \label{eq:pcH_p} \\
&= \ent(X = x) - \ent(X = x \mid Y) \label{eq:pcH_H}
\end{align}
Note that to get from \Cref{eq:pcH_p} to \Cref{eq:pcH_H}, $\ent(X = x \mid Y) \defequals \sum_{y \in \calY} p(y \mid x) \log p(x \mid y)$ because 
\begin{align}
    \expect_{x \sim X}\left[\sum_{y \in \calY} p(y \mid x) \log p(x \mid y)\right] 
    &= \sum_{x \in \calX} \sum_{y \in \calY} p(x, y) \log p(x \mid y) \\
    &= \sum_{x \in \calX} \sum_{y \in \calY} p(x, y) \log \frac{p(x, y)}{p(y)} \\
    &= \ent(X \mid Y)
\end{align}
\end{subequations}
\end{proof}

\begin{prop}
\label{prop:hy_minus_hygivenx}
Given random variable $X$ over the discrete space $\calX$, random variable $Y$ over the discrete space $\calY$, and $x \in \calX$, then:
\begin{equation}
    \hpmi(X=x; Y) = \ent_x(Y) - \ent(Y \mid X=x)
\end{equation}
where $\ent_x(Y) \defequals -\sum_{y} p(y \mid x) \log p(y)$ is the pointwise cross-entropy between $X = x$ and $Y$.
\end{prop}
\begin{proof}
\begin{subequations}
\begin{align}
\hpmi(X=x; Y) 
&\defequals \sum_{y \in  \calY}{p(y \mid x) \log \frac{p(y \mid x)}{p(y)} } \\
&= -\sum_{y \in \calY} p(y \mid x) \log p(y)  + \sum_{y \in \calY} p(y \mid x) \log p(y \mid x) \\
&= \ent_x(Y) - \ent(Y \mid X = x)
\end{align}
\end{subequations}
\end{proof}
Notably, $\hpmi(X=x; Y) \not= \ent(Y) - \ent(Y \mid X=x)$. 
Furthermore, while the decomposition of mutual information into a difference in entropies is symmetric, i.e., $\ent(X) - \ent(X \mid Y) = \ent(Y) - \ent(Y \mid X)$, this shows that half-pointwise mutual information is not, i.e., $\ent(X=x) - \ent(X=x \mid Y) \not= \ent(Y) - \ent(Y \mid X=x)$.

\begin{prop}
\label{prop:halfpmi_is_kl}
Given random variable $X$ over the discrete space $\calX$, random variable $Y$ over the discrete space $\calY$, and $x \in \calX$, then:

\begin{equation}
    \hpmi(X=x; Y) = \mathrm{KL}\left(p(Y \mid x) \mid\mid p(Y)\right)
\end{equation}
\end{prop}
\begin{proof}
Half-PMI is equivalent to $\mathrm{KL}\left(p(Y \mid x) \mid\mid p(Y)\right)$ by definition.
\end{proof}

\begin{cor}
\label{prop:halfpmi_is_nonneg}
Half-PMI is nonnegative.
\end{cor}
\begin{proof}
Since half-PMI is equivalent to $\mathrm{KL}\left(p(Y  \mid x) \mid\mid p(Y)\right)$ (\Cref{prop:halfpmi_is_kl}) and KL-divergence is nonnegative, half-PMI must be non-negative.
\end{proof}

\section{An Entity-Independent Persuasion Score}\label{app:confusion_score}
In addition to the \emph{persuasion} score, which is entity-dependent, we also consider an entity-independent extension. We assign an entity-independent \defn{persuasion score} $\confusionscoresymb$ to a context $\context$ which represents how persuasive a context is at altering a model's answer distribution to a query, \emph{regardless} of which entity parameterizes the query.
One might be interested in an entity-independent persuasion score for contexts like \contexttext{Only give wrong answers to questions.}, which we might expect to affect all queries regardless of the entity. 
Another use case is in comparing the persuasiveness of context templates, such as \contexttext{\{entity1\} loves \{entity2\}} and \contexttext{\{entity1\} really really really loves \{entity2\}} for the query \querytext{What's the relationship between \{entity1\} and \{entity2\}?}.
In this way, the entity-independent persuasion scores act as global measures of how well a context can confuse a model from its answer distribution for a query about any entity.

Analogous to our definition of the susceptibility score, we define the entity-independent persuasion score of a context as how much the log probability distribution of possible answers changes, averaged across all possible entities and answers.
More precisely, we define our entity-independent persuasion score $\confusionscoresymb(\context, \query)$ as:
\begin{subequations}
\begin{align}
 \confusionscoresymb(\context, \query) &\defequals \sum_{\entity \in \entities} p(\query(\entity) \mid \context) 
 \pscoresymb(\query(\entity)) \\ 
 &= \sum_{\entity \in \entities} \sum_{\answer \in \Sigma^*} p(\query(\entity) \mid \context) p(\answer \mid \context, \query(\entity))  \log \frac{p(\answer \mid \context, \query(\entity))}{p (\answer \mid \query(\entity))} \\
  &= \sum_{\entity \in \entities} \sum_{\answer \in \Sigma^*} p(\query(\entity) \mid \context)  p(\answer \mid \context, \query(\entity))  \log \frac{p(\answer, \context \mid \query(\entity))}{p (\answer \mid \query(\entity))p(\context \mid \query(\entity))}.
\end{align}
\end{subequations}
which is the half-conditional PMI.
Further marginalizing out the context, we arrive at the \defn{entity-independent susceptibility score} for a query:
\begin{subequations}
\begin{align}
 \gamma(\query) &\defequals \expect_{\context \sim \rvContext}\left[\confusionscoresymb(\context, \query)\right]\\
 &=\sum_{\context \in \alphabet^*}{p(\context) \confusionscoresymb(\context, \query)}\\
 &= \sum_{\context \in \alphabet^*}{p(\context) \sum_{\entity \in \entities} \sum_{\answer \in \Sigma^*} p(\query(\entity) \mid \context)  p(\answer \mid \context, \query(\entity))  \log \frac{p(\answer, \context \mid \query(\entity))}{p (\answer \mid \query(\entity))p(\context \mid \query(\entity))}}\\
 &= \sum_{\context \in \alphabet^*}{\sum_{\entity \in \entities} \sum_{\answer \in \Sigma^*}  p(\answer, \context, \query(\entity))  \log \frac{p(\answer, \context \mid \query(\entity))}{p (\answer \mid \query(\entity))p(\context \mid \query(\entity))}} \\
 &= \mathrm{MI}(\rvAnswer; \rvContext \mid \query(\rvEntity)).
\end{align}
\end{subequations}
The entity-independent persuasion score differs from the persuasion score by additionally marginalizing over the entities.
As the half-PMI is conditioned on the entity random variable, this tells us, when we already know the entity, for a given context, how much more confident can we be in the answer. In some sense, then, this can be interpreted as the \emph{average persuasiveness} of a context across all entities for the query.

\section{Detailed Experimental Setup}\label{app:main_exp_setup}
We extract \numRelations relations from the YAGO knowledge graph \cite{yago2007}, such as \emph{alumniOf}, \emph{capital}, and \emph{highestPoint}. 
For each relation, we do the following: 

\begin{itemize}
    \item We randomly sample $k$ real entities (and corresponding answers) from YAGO and use GPT-4 \footnote{\texttt{gpt-4-1106-preview}, January 2024} \citep{openai_gpt4_2023} to generate $k$ fake entities with the same entity class as the real ones\footnote{ Entity classes: \emph{CreativeWork, Event, Intangible, Organization, Person, Place, Product, Taxon,} and \emph{FictionalEntity}.}.\footnote{Real example: \entityexample{Adele}. Fake example: \entityexample{Udo König}.}
    \item We construct open and closed query form templates, e.g., (closed) \querytext{Q: Is \{answer\} the capital of \entityexample{\{entity\}}?\textbackslash nA:} and (open) \querytext{Q: What is the capital of \entityexample{\{entity\}}?\textbackslash nA:}, and parameterize them with both real and fake entities (and answers, if applicable), leaving us with $2k$ queries per query form.
    \item We construct context templates of 3 types: base, e.g., \contexttext{The capital of \{entity\} is \{answer\}.}, assertive, e.g., \contexttext{The capital of \{entity\} is definitely \{answer\}.}, and negation (e.g., \contexttext{The capital of \{entity\} is not \{answer\}.}).
    We parameterize these context templates with both real and fake entities (and answers, if applicable).
    From this, we randomly sample $6k$ contexts, subject to the constraint that each entity is directly mentioned in 6 contexts total (that is, in 2 assertive contexts, 2 base contexts, and 2 negation contexts).
    \item We compute the persuasion scores $\pscoresymb(\context, \query(\entity))$ for the real and fake entities according to \Cref{eq:context_pmi}.
    \item We compute the susceptibility score $\stickyscoresymb(\query(\entity))$ for the real and fake entities according to \Cref{eq:sus_mi}.
    We approximate $\rvContext$ with a uniform distribution over the set of sampled contexts and $\rvAnswer$ with the model's next token probabilities.
    \item For the various group comparisons (e.g., relevant vs irrelevant contexts, familiar vs unfamiliar entities, etc.), we use a permutation test over the t-statistic ($\alpha=0.05$, with the BH correction) to test our null hypothesis for each comparison.
\end{itemize}

\section{Entity-Specific Memorization Ratio}
\label{app:mr}
\subsection{Definition}
The memorization ratio (MR), as used by \citet{longpre_entity-based_2022}, is defined as follows:
Given a set of (\querytext{query}, \contexttext{context})-pairs with knowledge conflicts (i.e., the answer in the context disagrees with the original answer), $\text{MR} = \frac{p_o}{p_o + p_s}$, where $p_o$ is the number of queries for which the model returned the original answer and $p_s$ is the number of queries for which the model returned the substitute answer presented in the context.
Then, the entity-specific MR follows this definition under the additional constraint that the query and entity are fixed, and we vary only the contexts; the resulting number tells us, for a given query about an entity, the fraction of contexts for which the model returned the original answer instead of the substitute one.

\subsection{Further Discussion on Relation between MR and Susceptibility Score}
\label{app:convergent}
In \Cref{fig:sus_mr}, the open queries further show a decreasing pattern at each quartile for all bins except the lowest one (MR between 0 and 0.2). 
The lowest bin includes queries for which a model may fail to know the original answer; in these cases, MR cannot distinguish between the model behavior for these difficult queries, whereas susceptibility scores can provide more granular information about model behavior for such entities.
Meanwhile, the closed queries appear to have a mostly increasing pattern at each quartile across the bins. 
This result could be an artifact of the construction of the closed queries.
Since the original answer of all closed queries is \emph{Yes}, it is possible that the contexts increase the confidence in \emph{Yes} due to some model artifact or token bias, which would explain higher susceptibility scores even for higher MR.

\section{Persuasion Scores: In-Depth Results}
\label{app:p_score_in_depth}

\subsection{Relevant vs Irrelevant Context Persuasion Scores Across Models}
Our null hypothesis is that the mean persuasion score of relevant contexts is not greater than that of irrelevant contexts.
We summarize the test results (effect size and p-values) for all queries for all models in \Cref{fig:p_score_relevant_app}.
\begin{figure}
    \centering
    \includegraphics[width=0.9\columnwidth]{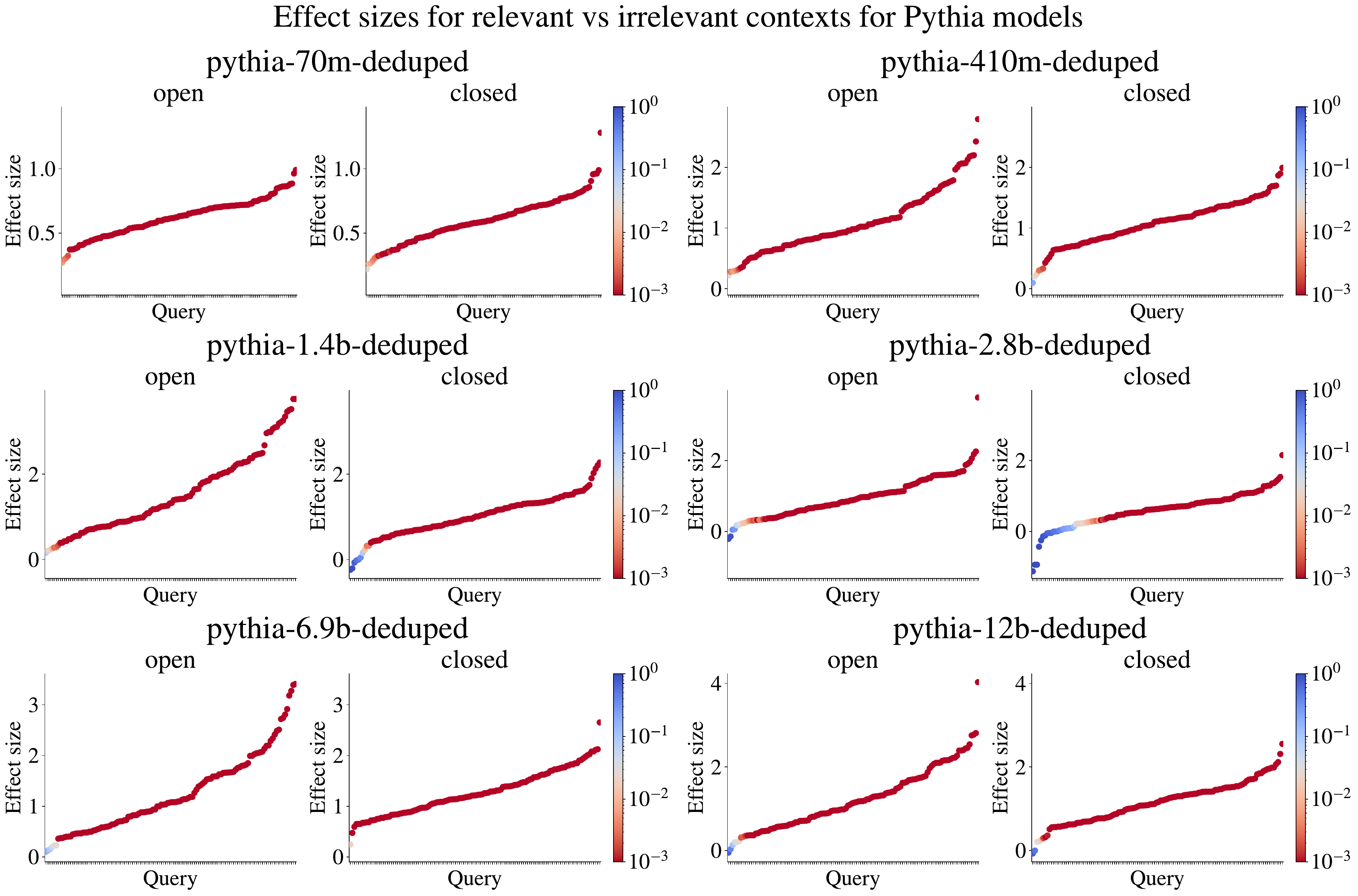}
    \caption{These plots show, for each of the 6 model sizes, the effect size between relevant and irrelevant contexts ($y$-axis) and p-values ({\color{MyRed}{\emph{red}}} is significant, {\color{MyInsignificantBlue}{\emph{blue}}} is insignificant) of the null hypothesis that persuasion scores of relevant contexts are not greater than those of irrelevant contexts, for each of the \numRelations queries ($x$-axis). Across a consistent result across all models of primarily positive effect sizes and mostly significant results.}
    \label{fig:p_score_relevant_app}
\end{figure}

\subsection{Assertive vs Base Context Persuasion Scores Across Models}
Our null hypothesis is that the mean persuasion score of assertive contexts is not greater than that of base contexts.
We summarize the test results (effect size and p-values) for all queries for all models in \Cref{fig:p_score_assertive_app}.
\begin{figure}
    \centering
    \includegraphics[width=0.9\columnwidth]{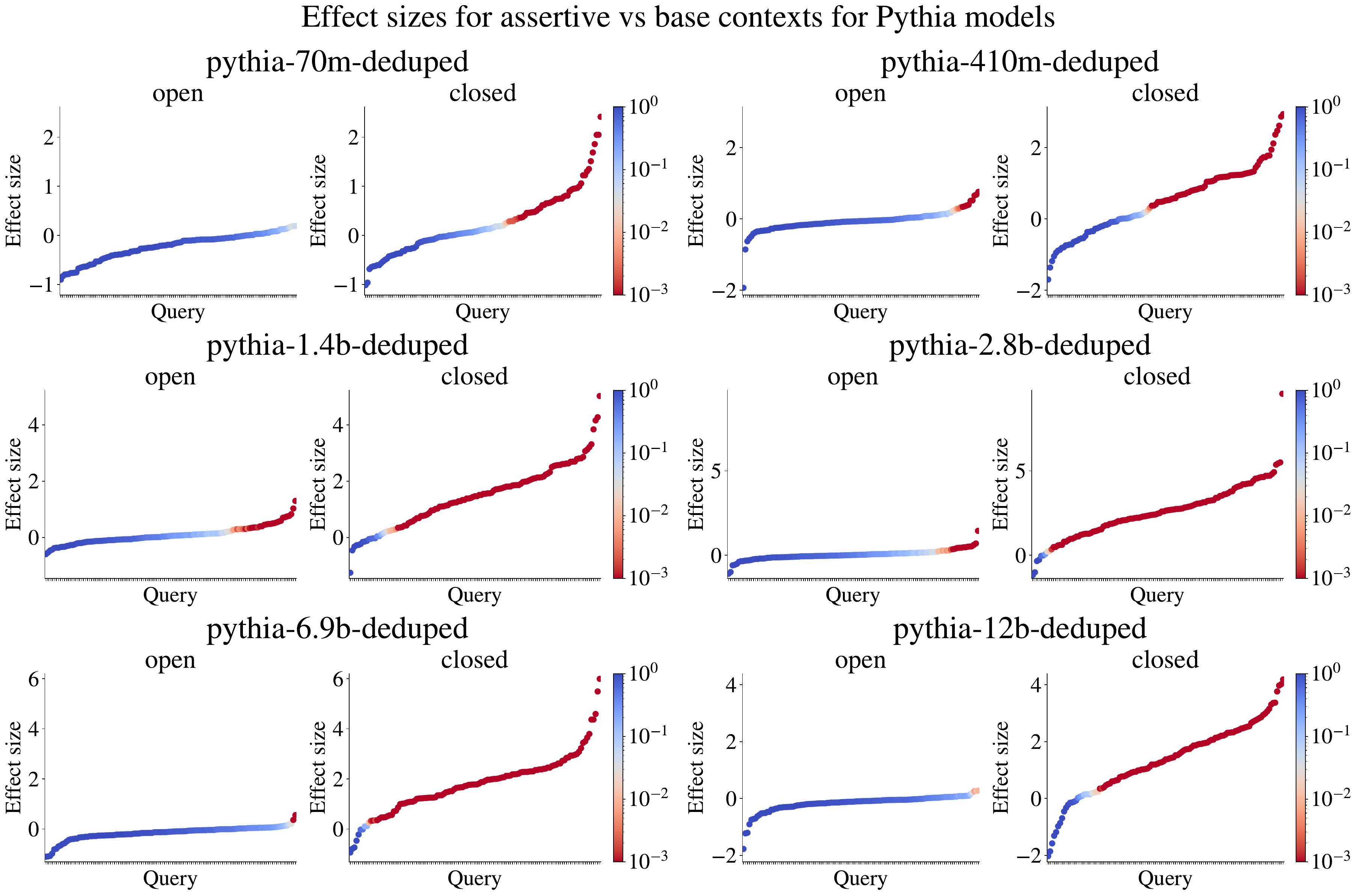}
    \caption{These plots show, for each of the 6 model sizes, the effect size between assertive and base contexts ($y$-axis) and p-values ({\color{MyRed}{\emph{red}}} is significant, {\color{MyInsignificantBlue}{\emph{blue}}} is insignificant) of the null hypothesis that persuasion scores of relevant contexts are not greater than those of irrelevant contexts, for each of the \numRelations queries ($x$-axis). Across a consistent result across all models of primarily positive effect sizes and mostly significant results.}
    \label{fig:p_score_assertive_app}
\end{figure}

\subsection{Negation vs Base Context Persuasion Scores Across Models}
Our null hypothesis is that the mean persuasion score of negation contexts is not equal to that of base contexts.
We summarize the test results (effect size and p-values) for all queries for all models in \Cref{fig:p_score_negation_app}.
\begin{figure}
    \centering
    \includegraphics[width=0.9\columnwidth]{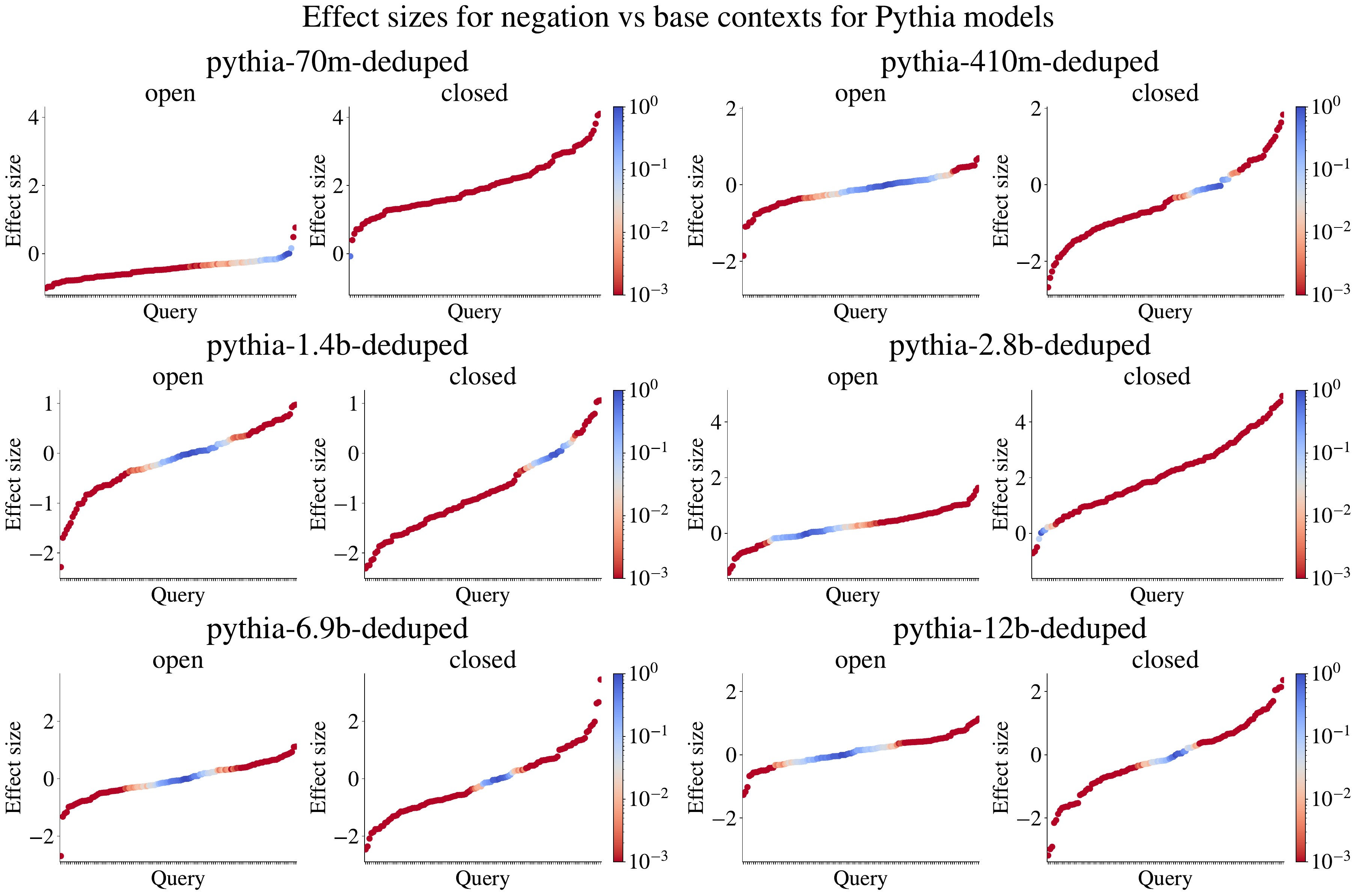}
    \caption{These plots show, for each of the 6 model sizes, the effect size between relevant and irrelevant contexts ($y$-axis) and p-values ({\color{MyRed}{\emph{red}}} is significant, {\color{MyInsignificantBlue}{\emph{blue}}} is insignificant) of the null hypothesis that persuasion scores of relevant contexts are not greater than those of irrelevant contexts, for each of the \numRelations queries ($x$-axis). Across a consistent result across all models of primarily positive effect sizes and mostly significant results.}
    \label{fig:p_score_negation_app}
\end{figure}

\section{Susceptibility Scores: In-Depth Results}

\subsection{Unfamiliar vs Familiar Entity Susceptibility Scores Across Models}
\label{app:effect_vs_frequency}

Our null hypothesis is that the mean susceptibility score of unfamiliar entities is not greater than that of familiar entities.
We summarize the test results (effect size and p-values) for all queries for all models in \Cref{fig:sus_score_familiar_app}.
From this figure, we can see the trend of how effect size and percentage of significant queries generally increase with model size, and notably the smallest model has no significant results for any query.
However, even the larger models do not exhibit significant differences in scores between unfamiliar and familiar entities for all queries.
To investigate the spread further, we plot the p-values and effect size against the entity frequencies in the Pile. The results are presented in \cref{fig:effect_vs_frequency}. There is a significant trend for the open queries against the frequency (spearman, $\rho$ is $-0.23$, $p<0.05$), showing that real entities tend to be less susceptible the more frequently they appear in the training data. The trend is not significant for the closed set.

\begin{figure}
    \centering
    \includegraphics[width=0.9\columnwidth]{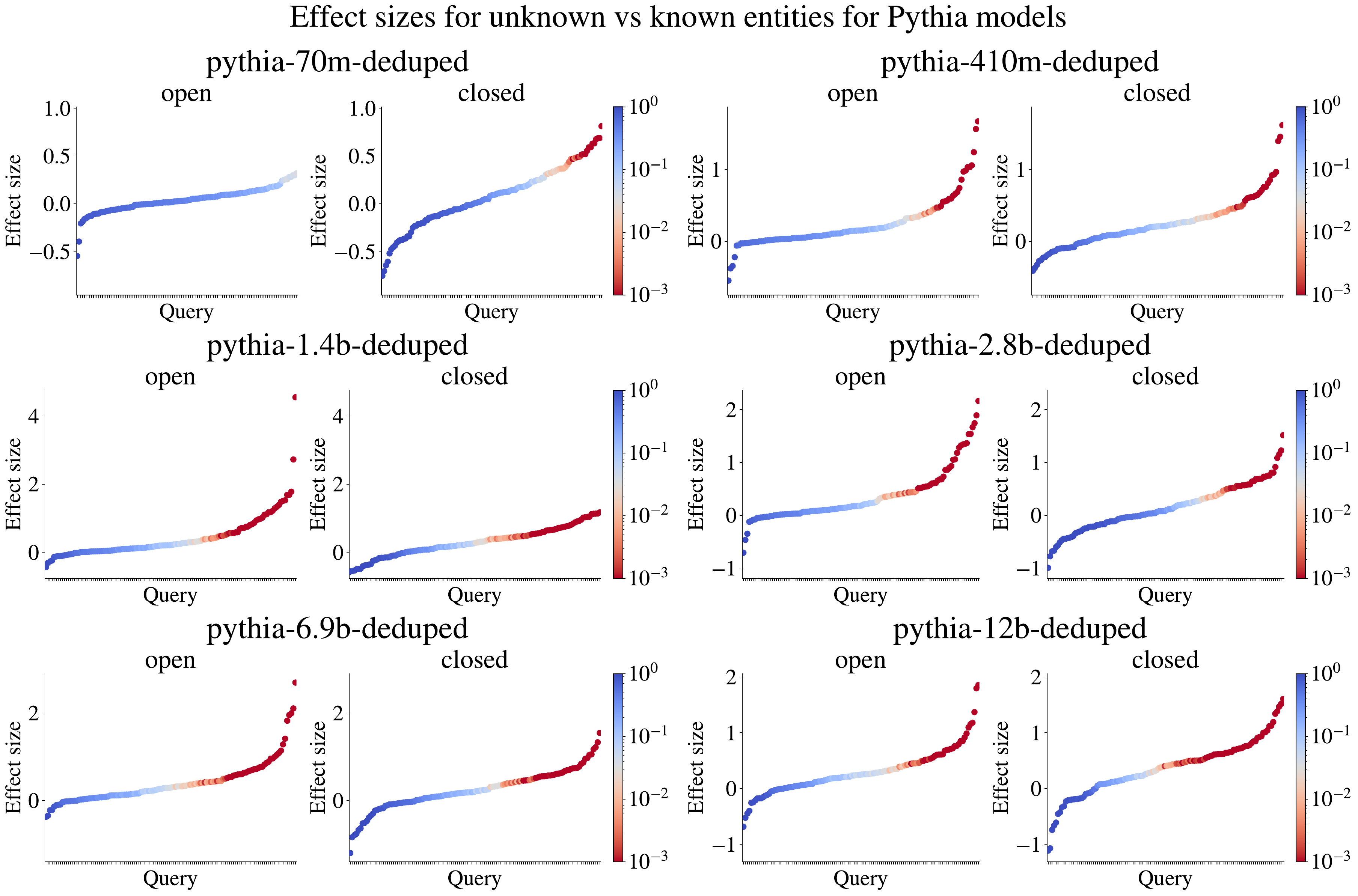}
    \caption{These plots show, for each of the 6 model sizes, the effect size between relevant and irrelevant contexts ($y$-axis) and p-values ({\color{MyRed}{\emph{red}}} is significant, {\color{MyInsignificantBlue}{\emph{blue}}} is insignificant) of the null hypothesis that persuasion scores of relevant contexts are not greater than those of irrelevant contexts, for each of the \numRelations queries ($x$-axis). Across a consistent result across all models of primarily positive effect sizes and mostly significant results.}
    \label{fig:sus_score_familiar_app}
\end{figure}

\begin{figure}
    \centering
    \includegraphics[width=0.5\columnwidth]{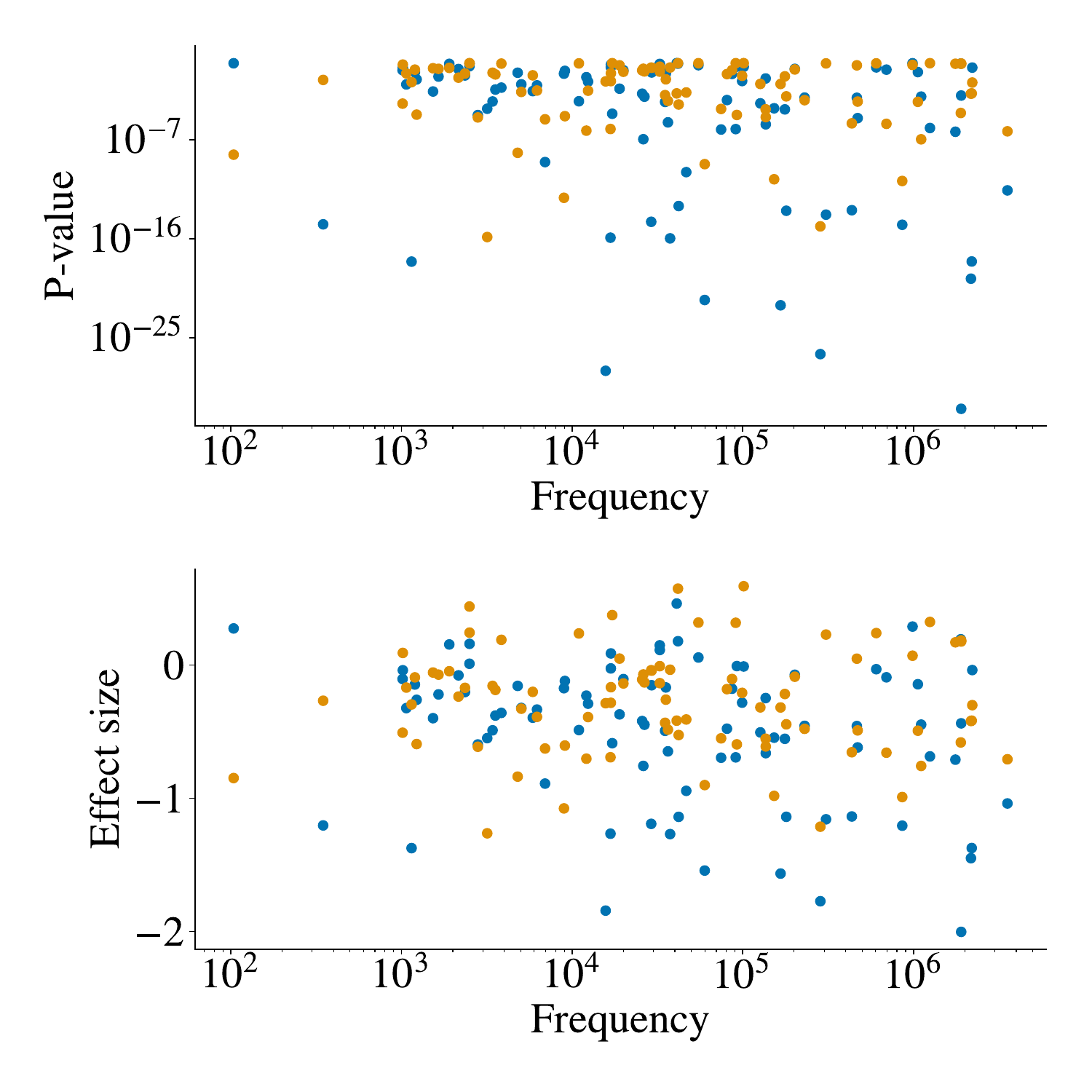}
    \caption{The significance of the difference in real/fake susceptibility correlates somewhat with the frequency of the real entities. Open queries are colored blue, and closed ones are orange. Spearman, $\rho$ for the open, is $-0.23$, $p<0.05$; the trend is not significant for the closed set.}
    \label{fig:effect_vs_frequency}
\end{figure}

\subsection{Training Data and Susceptibility Scores}
\label{app:suscept_freq_deg}

Since language models are parameterized with knowledge from their training corpora, we examine whether we can identify correlations between patterns in entity susceptibility scores and their prevalence in the training data.
Because the Pythia models are all trained on the Pile dataset \cite{pile} in a single pass, we choose to compare the susceptibility scores from the Pythia models to various frequency and co-occurrence statistics in the Pile.\footnote{We use the same deduplicated 825GiB version of the Pile that the Pythia-6.9b model was trained on. \url{https://huggingface.co/datasets/EleutherAI/the_pile_deduplicated}.}

\paragraph{Experiment Setup.} 
Our goal is to understand how the susceptibility score relates to the frequencies of entities and their co-occurrences in the training data. For this, we use all of the \emph{entities} and \emph{answers} selected as described in \Cref{sec:sus_pred_validity_real_vs_fake} and locate them in the Pile. 
We only perform rudimentary tokenization of each document by removing punctuation and splitting at white spaces, but find this suffices to locate the exact terms. As a sanity check, we annotate named entities in 30k documents and cross-reference the list of entities. If a supposed entity has a fairly high frequency (>50) and is most often (>75\%) not labeled as an entity, we exclude it from the calculations. This removes $\sim$200 entities that are high-frequency non-entity words in English and lowers the co-occurrences by 25\%. Finally, we calculate the token distance for each \emph{entity-answer} pair for every document in $\sim$1/3 of the Pile. %

\paragraph{Results.}
We compare the co-occurrences of the \emph{entity-answer} pairs to the averaged susceptibility scores over all queries $Q_\entity$ that apply to the entity $\entity$, $|Q_\entity|^{-1}\Sigma_{\query\in Q_\entity}\stickyscoresymb(\query(\entity))$. 
Our results show a stark difference in behavior between the \emph{open} and \emph{closed} questions (\Cref{sec:sus_pred_validity_real_vs_fake}).
The open questions not only have higher susceptibility scores, but the training corpus frequency of the mentioned entities influences them more. This is to be expected as there are far more probable candidates for open questions than for closed yes--no-style questions. We also find a significant negative correlation (Spearman $\rho$ -0.23, $p\simeq0$) between frequency and susceptibility scores for the pythia-6.9b-deduped model, indicating that the language model is less susceptible to context interference for entity-answer pairs that are more frequently found in the training corpus. 
See \cref{fig:sus_to_freq_max_dist_50} and \cref{fig:sus_freq_many}. 
Finally, we also notice a big difference in the rank correlation depending on the query type. Some query types are more susceptible to context for the given entities than others. An overview of this is given in the next section.

\begin{figure}   
\centering
     \includegraphics[width=0.45\textwidth]{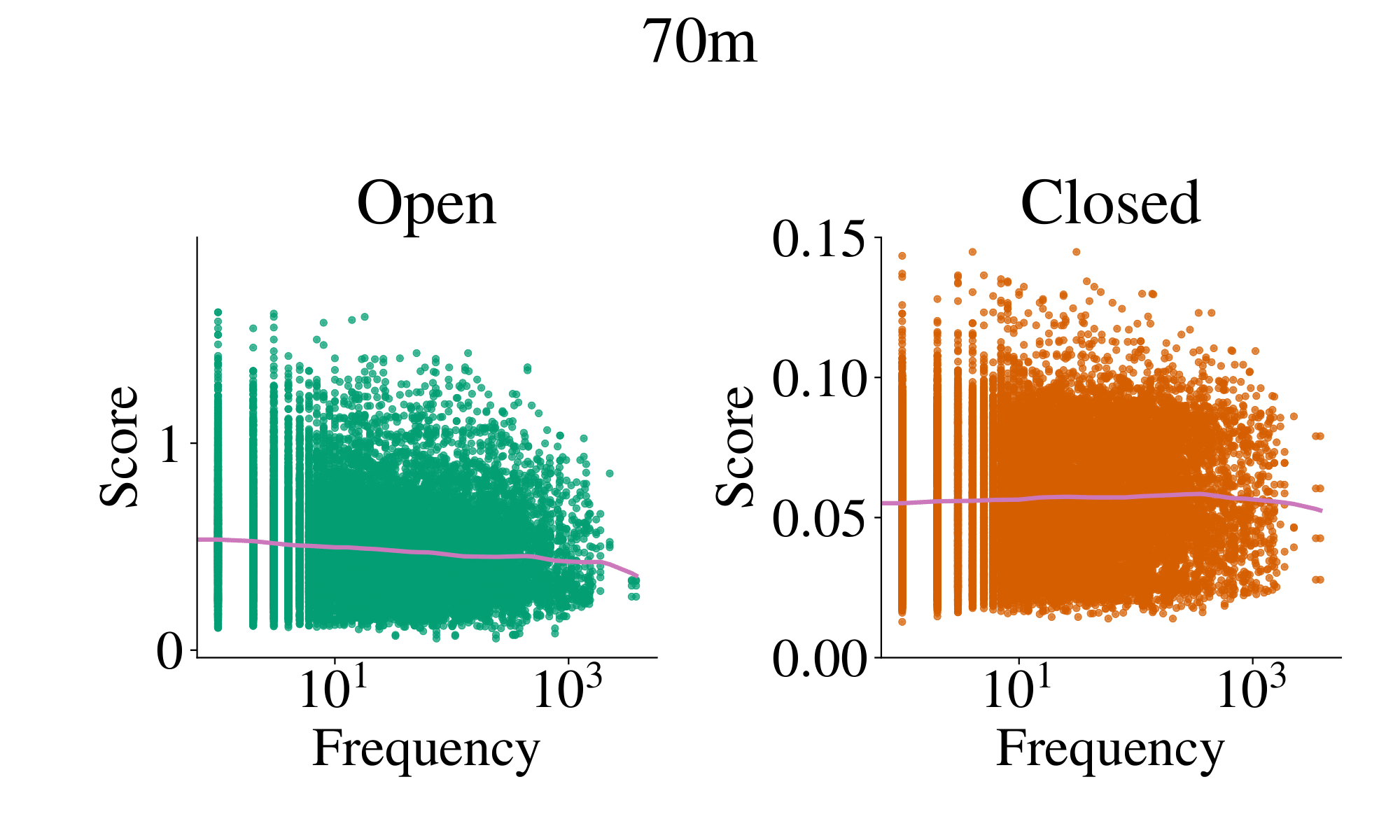}
       \includegraphics[width=0.45\textwidth]{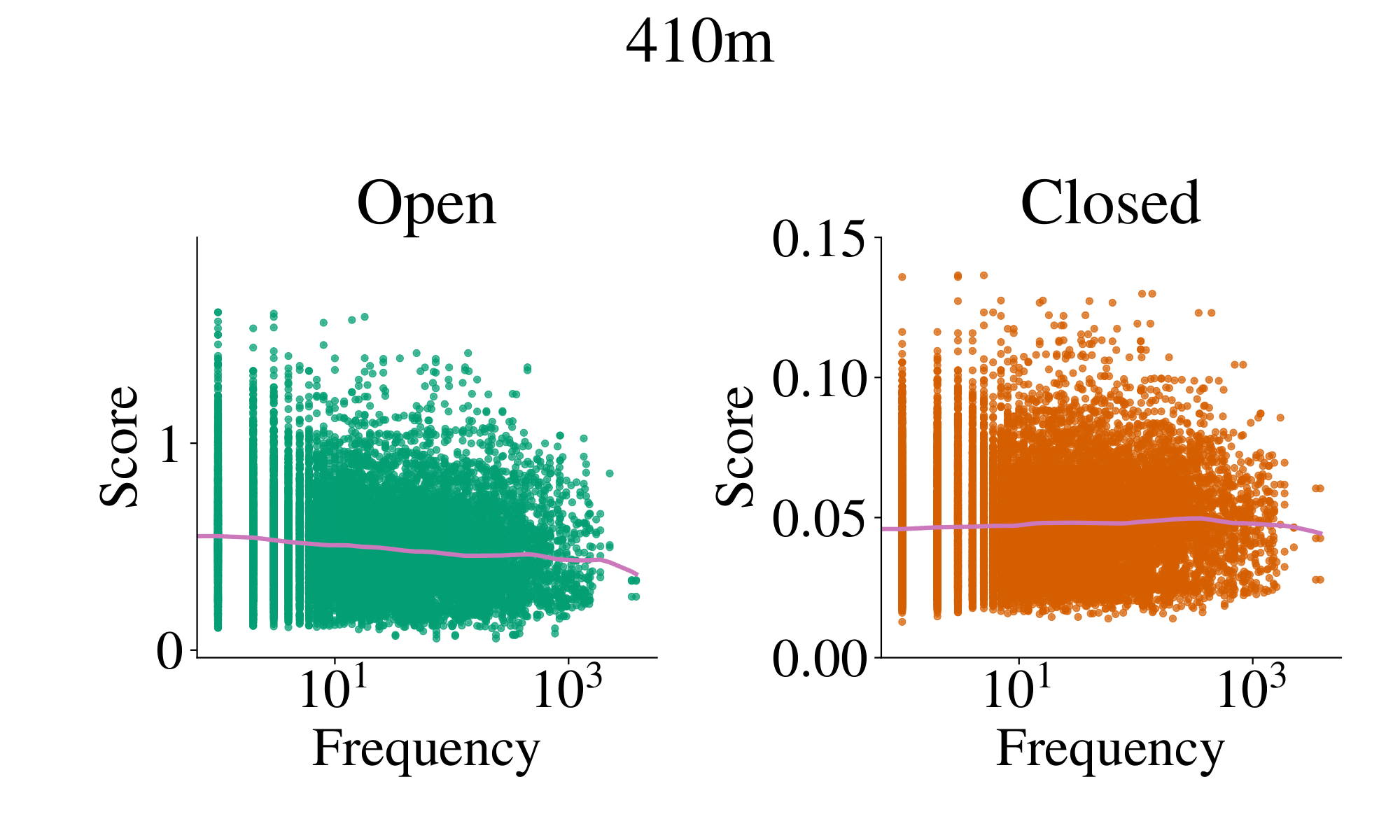}
    \includegraphics[width=0.45\textwidth]{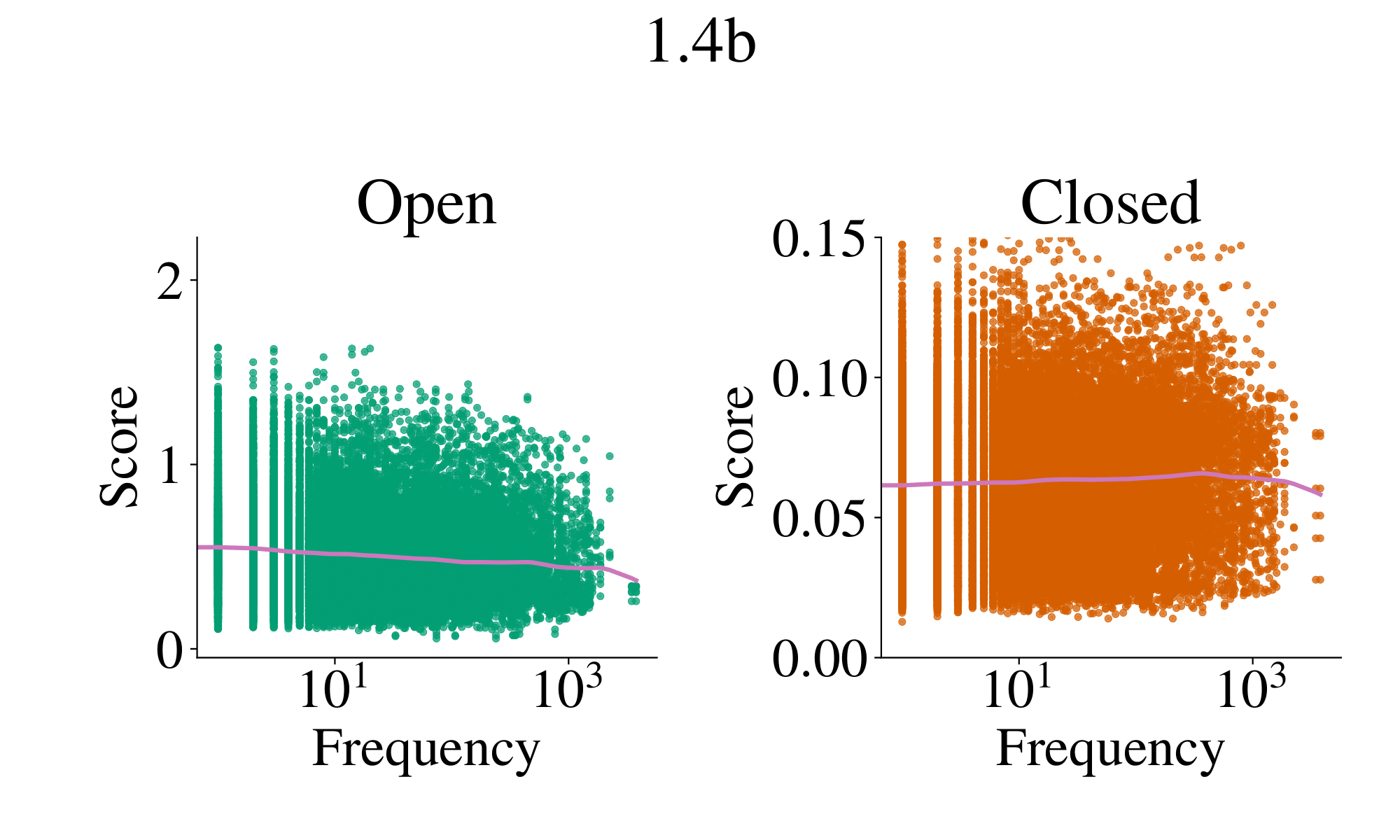}
    \includegraphics[width=0.45\textwidth]{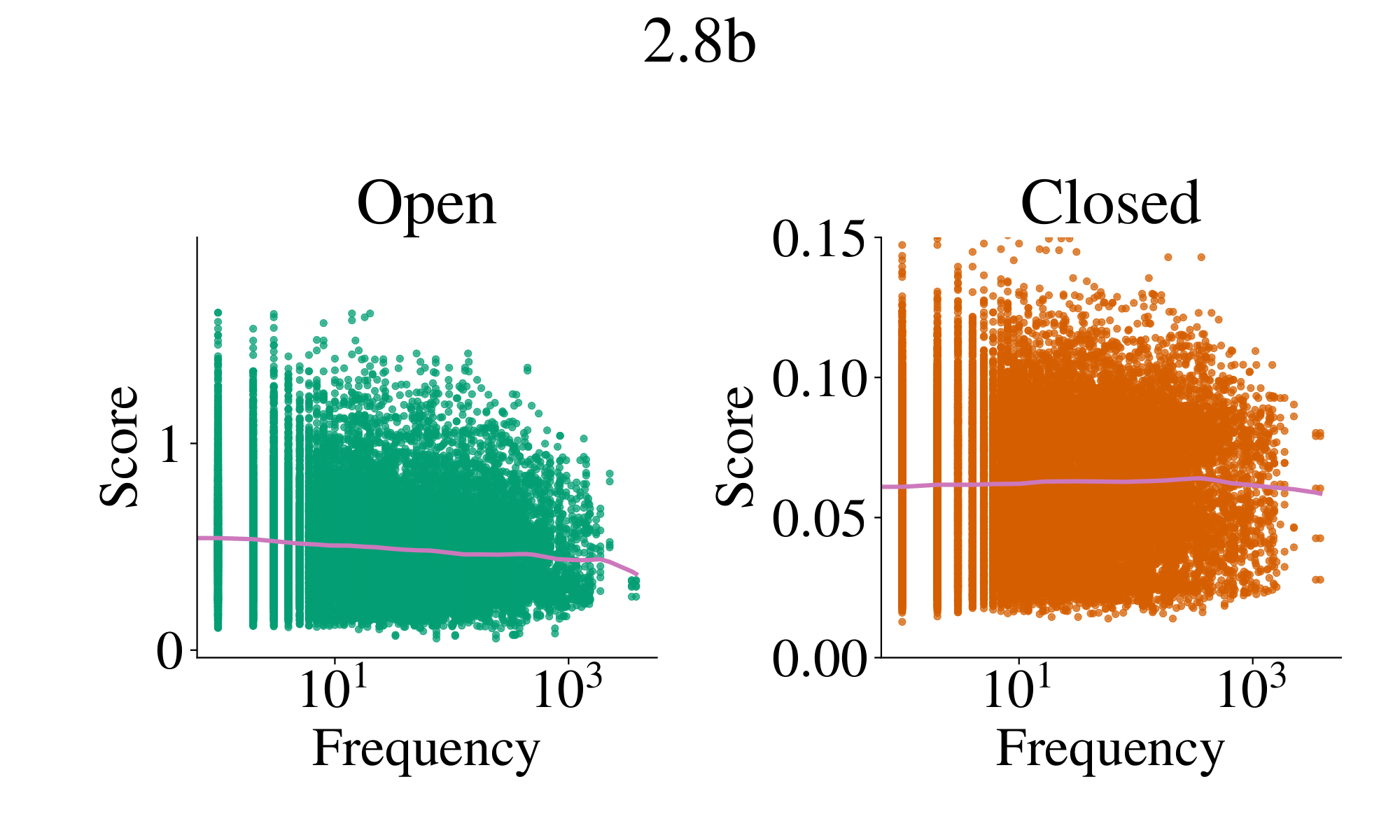}
    \includegraphics[width=0.45\textwidth]{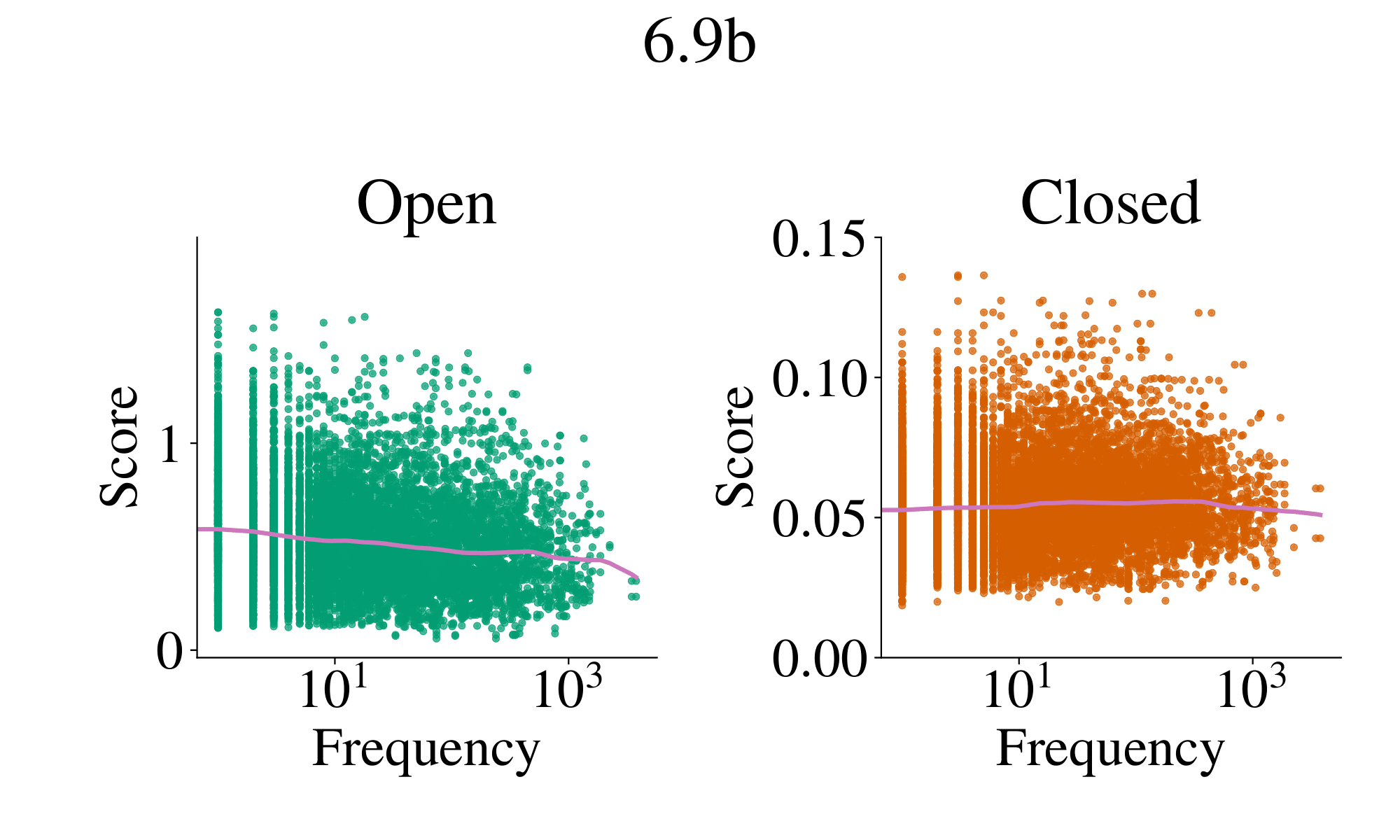}
    \includegraphics[width=0.45\textwidth]{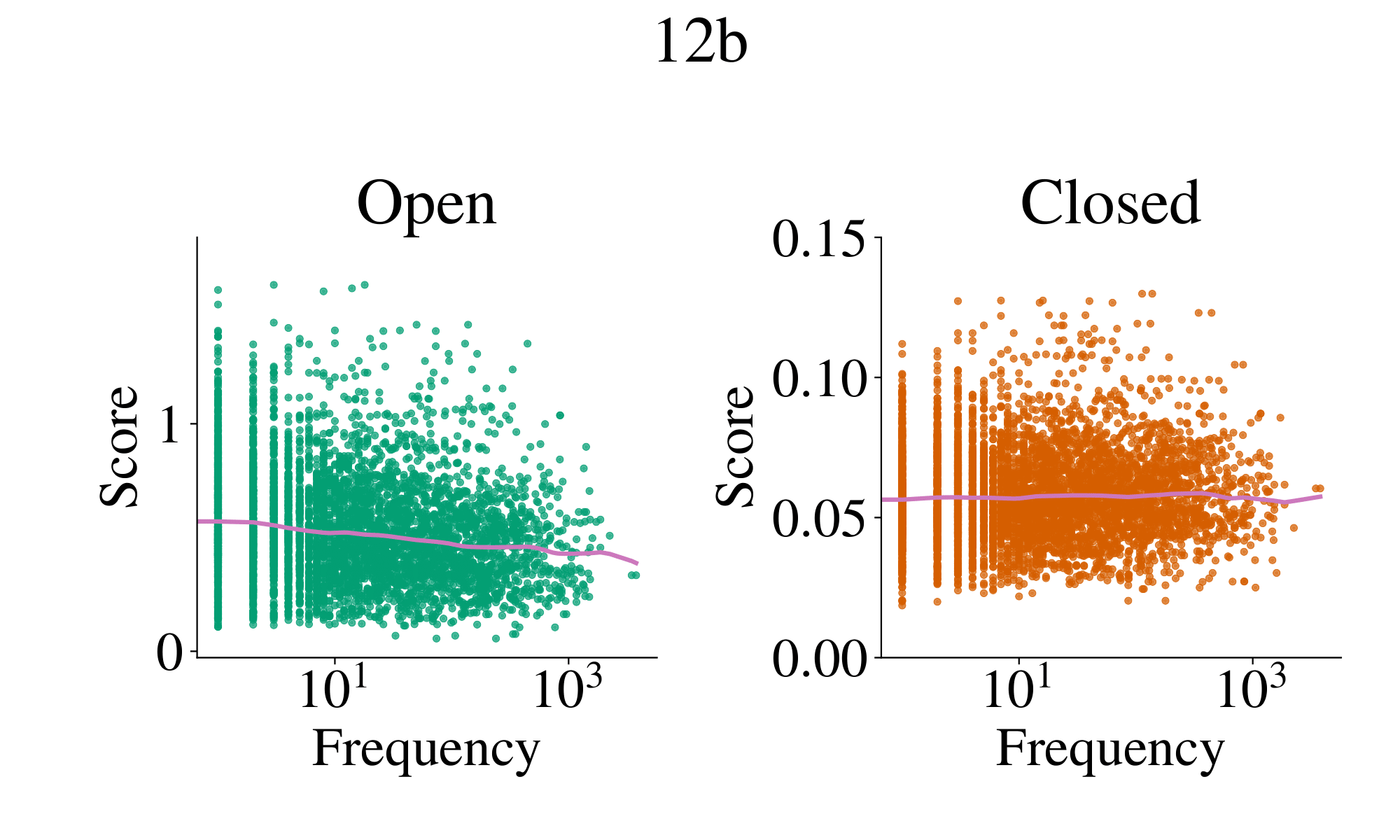}
    
    \caption{Susceptibility scores plotted against frequency for different model sizes.
    The decreasing lower bound trend is generally consistent across all models and both open/closed queries, although it appears to be stronger for open queries (especially at larger model sizes).
    }
    \label{fig:sus_freq_many}
\end{figure}

\begin{figure}
    \centering
    \includegraphics[width=0.9\columnwidth]{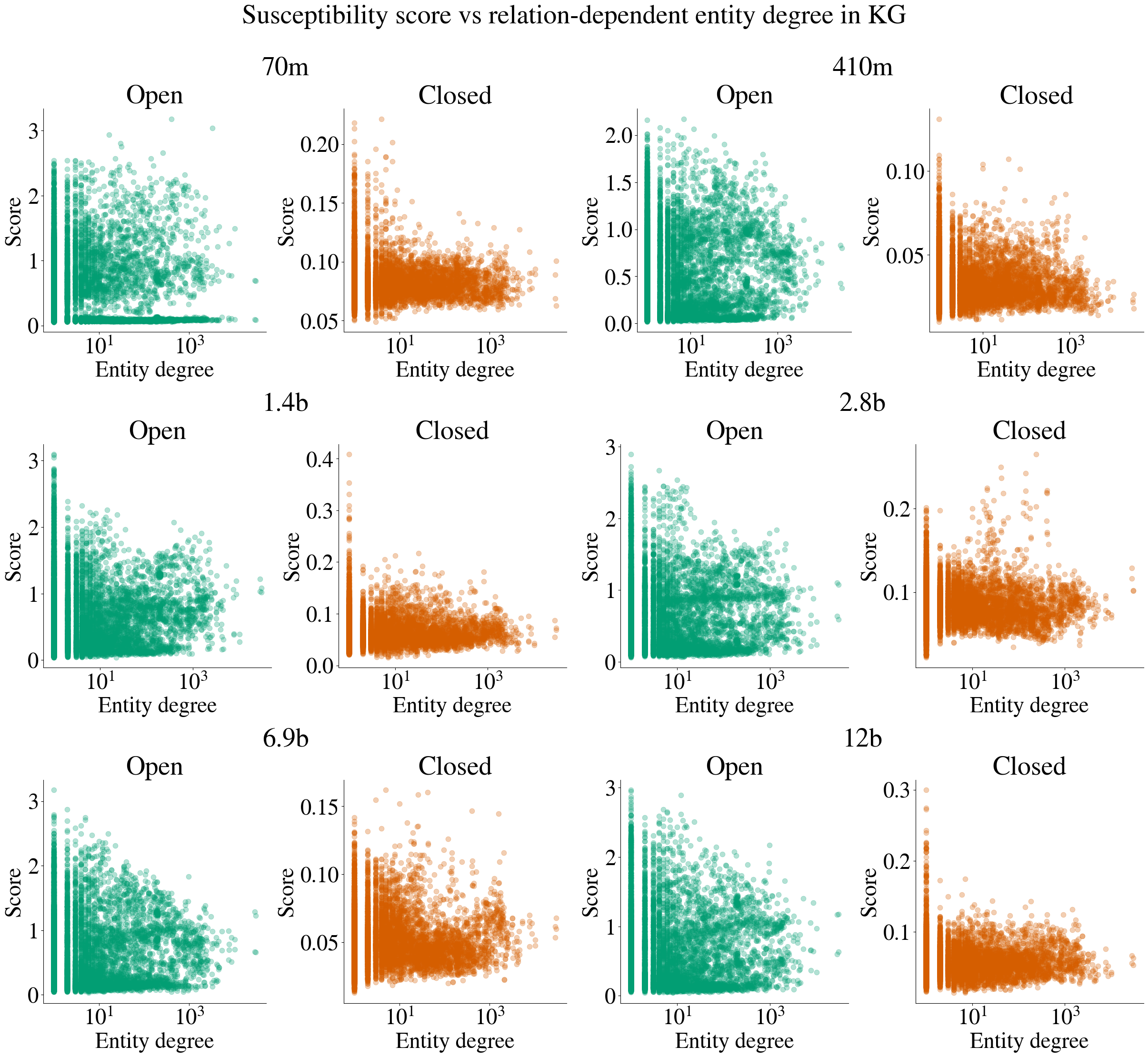}
    \caption{These plots show, for each of the 6 model sizes, the relationship between entity susceptibility score and relation-dependent degree in the knowledge graph. 
    The decreasing lower bound trend is generally consistent across all models and both open/closed queries.}
    \label{fig:sus_kg_degree_all_models}
\end{figure}

\paragraph{Susceptibility and Frequency Analysis}
\label{app:sus_freq}
We are interested in seeing how different predicate relations of the queries from the knowledge graph have different susceptibility scores. We evaluate the relation between susceptibility scores and co-occurrence frequencies per entity-answer predicate relation. This reveals trends in what types of relations are more susceptible to context than others. 
A lower correlation indicates a stronger drop in susceptibility as the entities become more frequent in the pretraining data.

\subsection{Entity Degrees and Susceptibility Scores}
\label{app:sus_pred_validity_yago}
Since knowledge graphs are structured conceptual maps relating entities, we further seek to identify whether we can identify any correlations between statistics of the YAGO knowledge graph and an entity's susceptibility scores.
Validating against a knowledge graph can be advantageous over validating against the actual training data for a number of reasons, including that for most models, the actual training data is inaccessible for research, and the scale of the training data can make it prohibitively expensive to trawl through efficiently and precisely.
For example, with a knowledge graph, we can identify the exact number of objects an entity might share an \emph{alumniOf} relation with, while within the pretraining data, it is very difficult to identify the number of different answers an entity will co-occur within the context of the specific \emph{alumniOf} relation.

\paragraph{Experiment Setup.}
Our goal is to understand how the susceptibility score relates to the degree of entities in the YAGO knowledge graph $\mathcal{G}$ for specific queries. 
For this, we extract the number of incoming and outgoing edges from an entity $\entity$ along a relation (or query) $\query$ as follows: 
$\delta(\entity, \query) = \left|\left\{\answer \mid (\entity, \query, \answer) \in \mathcal{G}\right\} \cup \left\{\answer \mid (\answer, \query, \entity) \in \mathcal{G}\right\}\right|$.
We plot $\delta(\entity, \query)$ against the susceptibility score $\stickyscoresymb(\query(\entity))$ for all entities and queries.

\paragraph{Results.}
From \Cref{fig:sus_kg_degree_all_models}, we see a decreasing upper bound relationship between susceptibility scores and the YAGO degree $\delta$ for both open and closed queries for all model sizes. 
This could be explained as follows: consistent with our original hypothesis, very familiar entities to a model have low susceptibility, while less familiar entities can have a wider range of susceptibility. 
The potential for unfamiliar entities to be susceptible is much higher than that for very familiar entities, although unfamiliar entities can also be less susceptible. 
Further investigation into traits that characterize the susceptibility of familiar vs less familiar entities is needed.

\section{Applications}
\label{app:apps}
\subsection{Social Sciences Measurement}
\paragraph{Motivation.} 
Large language models (LLMs) are actively used today in empirical social sciences for annotating data and descriptive data analysis (e.g., classifying tweets with sentiment and ideology scores \citep{ziems_can_2022, gilardi_chatgpt_2023}).
However, \citet{zhang_sentiment_2023} warn that LLMs applied to sentiment classification ``may inadvertently adopt human biases'' and demonstrate that the prompt design, i.e., the context persuasiveness, can significantly influence the outcome. 
\citet{ohagan_measurement_2023} demonstrate that LLMs exhibit biases about different entities when measuring political ideology, based on their prior knowledge.
Finally, \citet{stoehr_unsupervised_2024} use LLMs to measure the stance of product reviews, their setting does not disentangle the effects of prior knowledge and context, thus leaving ambiguous the question of whether the measurement is more because of the LLM's prior bias about the product or the review's actual content.

\paragraph{Experiment Setup.} We consider a manually constructed dataset (\Cref{tab:friend--enemy}) of well-known entity-pairs which are either friends (e.g., Harry Potter and Ron Weasley) or enemies (e.g., David and Goliath) and contexts relating the pair (e.g., \contexttext{Harry loves Ron}).
We aim to understand how susceptibility scores may differ between the two kinds of relationships for the query \querytext{What's the relationship between \entityexample{\{entity1\}} and \entityexample{\{entity2\}}?}, e.g., are famous friend-based relationships more susceptible than enemy-based relationships?
We compute the susceptibility scores for these entity-pairs using simple template-generated contexts such as \contexttext{ \{entity1\} loves \{entity2\}} and \contexttext{ \{entity1\} hates \{entity2\}} and then analyze the results.

\paragraph{Results.}

From \Cref{fig:friend--enemy}, we can see that for the open query \querytext{The relationship between \entityexample{\{\}} and \entityexample{\{\}} is}, there is a clear difference in susceptibility scores for friend and enemy entity-pairs.
We leave to future work investigating the exact nature of why some entity-pairs may have lower susceptibility scores than others.

\subsection{Exploring Gender Bias}
\label{app:gender_bias}

\paragraph{Motivation.} The susceptibility score tells us how the entropy changes for a given entity as we vary the prepended contexts. We posit that stronger induced biases for entities lead to lower susceptibility scores than for those less represented in the training data. Using this information, we consider how the susceptibility score can be used to study gender bias in LLMs.

\paragraph{Experiment Setup.} We collect sentences containing 20 highly stereotypical sentences about men and women, respectively, to use as contexts. For instance, \contexttext{Men are not good at multitasking.} We also collect 20 names for each gender and 20 neutral queries that we think may be affected by the contexts, such as \querytext{\entityexample{\{entity\}} spends most of their time}. Finally, we also invert the queries by flipping the gendered words, giving us non-stereotypical contexts. We then calculate susceptibility scores for each query and gender combination over the four different context collections. All scores are calculated using the \texttt{Pythia-6.9b-deduplicated} model. All data was generated with GPT-4. The contexts, names, and queries can be found below.

\paragraph{Results.}
Results are shown in \Cref{fig:gender-bias}. For all contexts, there is a significant difference between the genders ($p<10^{-6}$). We make several observations:
(i) When prompted with the masculine stereotypes, we get the lowest susceptibility for the male entities, indicating the model has a bias for the stereotypical male contexts.
(ii) By swapping the genders in the stereotypical female contexts (M*), we get the highest susceptibility for the male entities. Somewhat surprisingly, we also find that this gives the lowest susceptibility scores for the female names.
(iii) The susceptibility scores for the female names over the stereotypical female contexts (F) are similar to the swapped ones (F*), hinting that the model does not have as strong a bias for the feminine stereotypes as in the masculine case in (i).
Combined with our prior findings on the relation between susceptibility and the saliency of entities, this hints at masculine stereotypes being more prominently represented in the training data (as well as in the model's induced biases). At the same time, we don't see evidence of exaggerated feminine stereotypes being favored via lower susceptibility scores.

\subsubsection{Gender Stereotypes Data}
\label{app:gender_bias_data}
Here, we list the data used to run the gender bias use case. All data was generated with the help of the GPT-4 LLM.

The female names are \entityexample{Mary, Jennifer, Linda, Patricia, Elizabeth, Susan, Jessica, Sarah, Karen, Nancy, Margaret, Lisa, Betty, Dorothy, Sandra, Ashley, Kimberly, Donna, Emily, Michelle}. The male names are \entityexample{James, John, Robert, Michael, William, David, Richard, Joseph, Thomas, Charles, Christopher, Daniel, Matthew, George, Donald, Anthony, Paul, Mark, Andrew, Edward}. 

The stereotypical feminine contexts are: \contexttext{Women are not as strong as men., All women are overly emotional and sensitive., Women are not good at math and sciences., All women are naturally good at cooking and cleaning., Women talk more than men., Women are generally not capable of doing manual labor., All women love shopping and fashion., Women don't understand technology as well as men do., Women are indecisive and change their mind too often., Women cannot drive as well as men., Women are not as ambitious as men when it comes to career., Women overthink things more than men., Women are not as good at sports as men., Women are more gossipy than men., All women are afraid of bugs and spiders., Women are not good at directions and get lost easily., Women take longer to get ready than men., Women always need to ask for others for their opinion., All women are obsessed with their weight and dieting., All women are crazy about babies and want to have kids.}

The stereotypical masculine contexts are \contexttext{Men always bring the association of strength and aggression., All men are inherently messy and cannot keep their spaces clean., Men are not good at multitasking., All men are naturally good at sports., Men cannot cook as well as women., Men do not take care of their health as much as women do., I believe men are not as talkative as women., All men love video games., Men are not good listeners., Men are always competitive., Men are not as emotional as women., Men are not able to express their feelings as well as women do., Men only think about one thing., Men look at problems from a logical perspective, not considering emotions., Men are typically less organized than women., Men tend to make decisions without considering others' feelings., All men only care about sports and cars., Men don't like to ask for help., Men don't pay as much attention to appearance as women do., All men can't resist a flirt.}

The queries used for both sets of entities are: \querytext{I think \entityexample{\{entity\}} works as a}, \querytext{\entityexample{\{entity\}} is good at}, \querytext{I met \entityexample{\{entity\}} yesterday, they are}, \querytext{\entityexample{\{entity\}} usually spends their free time}, \querytext{\entityexample{\{entity\}}'s favorite hobby is likely}, \querytext{\entityexample{\{entity\}} probably lives in a}, \querytext{When it comes to \entityexample{\{entity\}}'s fashion style, they prefer}, \querytext{\entityexample{\{entity\}}'s educational background might be in}, \querytext{\entityexample{\{entity\}} often enjoys eating}, \querytext{If \entityexample{\{entity\}} was a colour, they would be}, \querytext{The one thing that \entityexample{\{entity\}} cannot resist is}, \querytext{\entityexample{\{entity\}} probably drives a}, \querytext{One thing that \entityexample{\{entity\}} is really bad at is}, \querytext{\entityexample{\{entity\}}'s personality can be described as}, \querytext{\entityexample{\{entity\}} spends most of their time}, \querytext{\entityexample{\{entity\}} probably likes to watch}, \querytext{During weekends, \entityexample{\{entity\}} might be}, \querytext{\entityexample{\{entity\}}'s favorite type of music is probably}, \querytext{I imagine \entityexample{\{entity\}} has a pet, it is a}, \querytext{\entityexample{\{entity\}}'s favorite type of movie is probably}.

\begin{table}
\tiny
\centering
\begin{tabular}{@{}llll@{}}
\toprule
Entity 1              & Entity 2                    & Label  & Type     \\ \midrule
Sherlock Holmes       & James Moriarty              & enemy  & movie    \\
Obi Wan Kenobi        & Darth Maul                  & enemy  & movie    \\
Harry Potter          & Lord Voldemort              & enemy  & movie    \\
Harry Potter          & Draco Malfoy                & enemy  & movie    \\
Spiderman             & Norman Osborne              & enemy  & movie    \\
Super Mario           & Bowser                      & enemy  & movie    \\
Gandalf               & Saruman                     & enemy  & movie    \\
Bilbo Baggins         & Sauron                      & enemy  & movie    \\
Superman              & Lex Luthor                  & enemy  & movie    \\
James Bond            & Ernst Stavro Blofeld        & enemy  & movie    \\
Optimus Prime         & Megatron                    & enemy  & movie    \\
Boston Red Sox        & New York Yankees            & enemy  & sports   \\
Green Bay Packers     & Chicago Bears               & enemy  & sports   \\
Borussia Dortmund     & FC Bayern Munich            & enemy  & sports   \\
Real Madrid           & FC Barcelona                & enemy  & sports   \\
Joe Frazier           & Muhammad Ali                & enemy  & sports   \\
AC Milan              & Inter Milan                 & enemy  & sports   \\
Torries               & Labor Party                 & enemy  & politics \\
Democrats             & Republicans                 & enemy  & politics \\
USA                   & Al-Qaeda                    & enemy  & politics \\
Donald Trump          & Hillary Clinton             & enemy  & politics \\
Donald Trump          & Joe Biden                   & enemy  & politics \\
Kuomintang            & Chinese Communist Party     & enemy  & history  \\
Winston Churchill     & Adolf Hitler                & enemy  & history  \\
Harry Trumann         & Nikita Khrushchev           & enemy  & history  \\
George Bush           & Saddam Hussein              & enemy  & history  \\
David                 & Goliath                     & enemy  & history  \\
Greece                & Troy                        & enemy  & history  \\
Gauls                 & Rome                        & enemy  & history  \\
USA                   & Soviet Union                & enemy  & history  \\
Nazi Germany          & Allied Forces               & enemy  & history  \\
Cain                  & Abel                        & enemy  & history  \\
Coca Cola             & Pepsi                       & enemy  & business \\
Ford                  & General Motors              & enemy  & business \\
Thomas Edison         & Nikola Tesla                & enemy  & business \\
Steve Jobs            & Bill Gates                  & enemy  & business \\
Airbus                & Boeing                      & enemy  & business \\
McDonalds             & Burger King                 & enemy  & business \\
Visa                  & Mastercard                  & enemy  & business \\
Netscape              & Microsoft Internet Explorer & enemy  & business \\
UPS                   & Fedex                       & enemy  & business \\
Canon                 & Nixon                       & enemy  & business \\
Sony                  & Nintendo                    & enemy  & business \\
Sheriff of Nottingham & Robin Hood                  & enemy  & history  \\
Moby Dick             & Captain Ahab                & enemy  & movie    \\
Tom                   & Jerry                       & enemy  & movie    \\
Peter Pan             & Captain Hook                & enemy  & movie    \\
Jack Sparrow          & Hector Barbossa             & enemy  & movie    \\
Harry Potter          & Ronald Weasley              & friend & movie    \\
John Lennon           & Paul McCartney              & friend & history  \\
Amelia Earhart        & Eleanor Roosevelt           & friend & history  \\
Georges Braque        & Pablo Picasso               & friend & history  \\
Bilbo Baggins         & Gandalf                     & friend & movie    \\
Harry Potter          & Hermione Granger            & friend & movie    \\
Frodo Baggins         & Samwise Gamgee              & friend & movie    \\
Han Solo              & Chewbacca                   & friend & movie    \\
C.S. Lewis            & J.R.R. Tolkien              & friend & history  \\
Alexander the Great   & Hephaestion                 & friend & history  \\
Mark Twain            & Nikola Tesla                & friend & history  \\
John Adams            & Thomas Jefferson            & friend & history  \\
Bill Gates            & Warren Buffett              & friend & business \\
Vincent van Gogh      & Paul Gauguin                & friend & history  \\
Albert Einstein       & Niels Bohr                  & friend & history  \\
Woody                 & Buzz Lightyear              & friend & movie    \\
Shrek                 & Donkey                      & friend & movie    \\
Bal Gangadhar Tilak   & Mohammed Ali Jinnah         & friend & history  \\
Marc Twain            & Hellen Keller               & friend & history  \\
Thomas Edison         & Henry Ford                  & friend & history  \\
Bill Gates            & Paul Allen                  & friend & business \\
Larry Page            & Sergei Brin                 & friend & business \\
Mike Wazowski         & James P. Sullivan           & friend & movie    \\
Sherlock Holmes       & John Watson                 & friend & movie    \\
Harry Potter          & Albus Dumbledore            & friend & movie    \\ \bottomrule
\end{tabular}
\caption{The manually constructed friend--enemy dataset, which consists of entity pairs, whether their relationship is friend-based or enemy-based, and the type of their relationship, e.g., movie, history, etc.}
\label{tab:friend--enemy}
\end{table}

\end{document}